\documentclass[twoside]{article}

\usepackage[accepted]{aistats2019}
\usepackage[utf8]{inputenc} 
\usepackage[T1]{fontenc}    
\usepackage{lmodern}
\usepackage{hyperref}       
\usepackage{url}            
\usepackage{booktabs}       
\usepackage{amsfonts}       
\usepackage{nicefrac}       
\usepackage{microtype}      
\usepackage{graphicx}
\usepackage{amsthm, amsmath, amssymb}
\usepackage{subcaption}
\usepackage[numbers, compress]{natbib}

\newtheorem{theorem}{Theorem}
\newtheorem{example}{Example}
\newtheorem{corollary}{Corollary}
\newtheorem{lemma}{Lemma}
\newtheorem*{definition}{Definition}

\def\psfancypar#1#2{\begingroup\def\par{\endgraf\endgroup\lineskiplimit=0pt}
               \setbox2=\hbox{\large\sc #2}
               \newdimen\tmpht \tmpht \ht2 \advance\tmpht by \baselineskip
               \font\hhuge=Times-Bold at \tmpht
               \setbox1=\hbox{{\hhuge #1}}
               \count7=\tmpht \count8=\ht1
               \divide\count8 by 1000 \divide\count7 by \count8
               \tmpht=.001\tmpht\multiply\tmpht by \count7
               \font\hhuge=Times-Bold at \tmpht
               \setbox1=\hbox{{\hhuge #1}}
               \noindent
                \hangindent1.05\wd1
               \hangafter=-2 {\hskip-\hangindent
               \lower1\ht1\hbox{\raise1.0\ht2\copy1}%
                \kern-0\wd1}\copy2\lineskiplimit=-1000pt}

\newcommand{\beq}{\begin{equation}}
\newcommand{\eeq}{\end{equation}}
\newcommand{\bqa}{\begin{eqnarray}}
\newcommand{\eqa}{\end{eqnarray}}
\newcommand{\bqn}{\begin{eqnarray*}}
\newcommand{\eqn}{\end{eqnarray*}}
\newcommand{\nn}{\nonumber}

\newcommand{\be}{\begin{enumerate}}
\newcommand{\ee}{\end{enumerate}}
\newcommand{\bi}{\begin{itemize}}
\newcommand{\ei}{\end{itemize}}
\newcommand{\bd}{\begin{description}}
\newcommand{\ed}{\end{description}}
\newcommand{\ba}{\begin{array}}
\newcommand{\ea}{\end{array}}
\newcommand{\bde}{\begin{definition}}
\newcommand{\ede}{\end{definition}}
\newcommand{\bex}{\begin{example}}
\newcommand{\eex}{\end{example}}

\def\boxit#1{\vbox{\hrule\hbox{\vrule\kern3pt
        \vbox{\kern3pt#1\kern3pt}\kern3pt\vrule}\hrule}}

\def\reals{ { {\rm  I \kern-0.15em R }  } }
\def\complex{ {\,{{\rm C} \kern-0.50em \raise0.20ex {  |}}\, }}

\def\0bf{{\bf 0}}
\def\1bf{{\bf 1}}
\def\2bf{{\bf 2}}
\def\3bf{{\bf 3}}
\def\4bf{{\bf 4}}
\def\5bf{{\bf 5}}
\def\6bf{{\bf 6}}
\def\7bf{{\bf 7}}
\def\8bf{{\bf 8}}
\def\9bf{{\bf 9}}

\def\Rbf{{\bf R}}

\def\Rxx{\Rbf_{\ssstyle X\kern-.1em X}}

\let\ssstyle=\scriptscriptstyle

\def\Kout{\setbox1=\hbox{\Huge\bf K}\hbox to
1.05\wd1{\hspace{.05\wd1}
}}
\def\Sout{\setbox1=\hbox{\Huge\bf S}\hbox to 1.05\wd1{\hspace{.05\wd1}
}}

\begin{document}
	
	%
	
	%
	\runningtitle{Universal Hypothesis Testing with Kernels}
	\runningauthor{Shengyu Zhu, Biao Chen, Pengfei Yang, Zhitang Chen}
	\twocolumn[
	\aistatstitle{Universal Hypothesis Testing with Kernels:\\Asymptotically Optimal Tests for Goodness of Fit}
	
	\aistatsauthor{Shengyu Zhu \And Biao Chen \And  Pengfei Yang \And Zhitang Chen}

\aistatsaddress{Huawei Noah's Ark Lab \\Hong Kong \And  Syracuse University \\Syracuse, NY  \And Cubist Systematic Strategies\\New York, NY \And Huawei Noah's Ark Lab \\Hong Kong } ]
	
\begin{abstract} 
 We characterize the asymptotic performance of nonparametric goodness of fit testing. The exponential decay rate of the type-II error probability is used as the asymptotic performance metric, and a test is optimal if it achieves the maximum rate subject to a constant level constraint on the type-I error probability. We show that two classes of Maximum Mean Discrepancy (MMD) based tests attain this optimality on $\mathbb R^d$, while the quadratic-time Kernel Stein Discrepancy (KSD) based tests achieve the maximum exponential decay rate under a relaxed level constraint. Under the same performance metric, we proceed to show that the quadratic-time MMD based two-sample tests are also optimal for general two-sample problems, provided that kernels are bounded continuous and characteristic. Key to our approach are Sanov's theorem from large deviation theory and the weak metrizable properties of the MMD and KSD. 
\end{abstract} 
\section{Introduction}
Goodness-of-fit tests play an important role in machine learning and statistical analysis. Given a model distribution $P$ and sample $x^n:=\{x_i\}_{i=1}^n$ originating from an unknown distribution $Q$, the goal is to decide whether to accept the null hypothesis that $Q$ matches $P$, or the alternative hypothesis that $Q$ and $P$ are different. Traditional (parametric) approaches may require  space partitioning or closed-form integrals \cite{Baringhaus1988,Beirlant1994,Bowman1993adaptive,Gyorfi1991}. They become computationally intractable to machine learning applications that involve high dimensional data and complicated models \cite{Koller2009,Salakhutdinov2015,Sutherland2017GeneandCrit}.

Recently, several efficient tests have been proposed based on Reproducing Kernel Hilbert Space (RKHS) embedding \cite{Muandet2017KME,SmolaHilbert}. One is to conduct a Maximum Mean Discrepancy (MMD) based two-sample test by drawing samples from the model distribution $P$ \cite{Lloyd2015ModelCrit}. A difficulty with this approach is to determine the number of samples drawn from $P$ relative to $n$, the sample number of the test sequence. Other tests are based on classes of Stein transformed RKHS functions \cite{Chwialkowski2016Goodness,Gorham2015Stein,Gorham2017measuringKSD,Liu2016GoodnessFit,Oates2017SteinRKHS},  where the test statistic is the norm of the smoothness-constrained function with the largest expectation under $Q$ and is referred to as the Kernel Stein Discrepancy (KSD). The KSD based tests only require knowing the density function of $P$ up to the normalization constant, and do not need to  compute integrals or draw samples. Additionally, constructing explicit features of distributions results in a linear-time goodness-of-fit test that is also more interpretable \cite{Jitkrittum2017linearGoodness}.

Motivated by their good performance in practice, this paper investigates the statistical optimality of these kernel based goodness-of-fit tests, a long-standing open problem in information theory and statistics \cite{Cover2006,Csiszar2004,Hoeffding1965}. Given distribution $P$, the hypothesis testing between $H_0: x^n\sim P$ and $H_1:x^n\sim Q$ can be extremely hard when $Q$ is arbitrary but unknown, as opposed to the simple case when $Q$ is known. With independent sample and a known $Q$, the type-II error probability of an optimal test vanishes exponentially fast w.r.t.~the sample size $n$, and the exponential decay rate coincides with the Kullback-Leibler Divergence (KLD) between $P$ and $Q$ (cf.~Lemma~\ref{lem:SteinLemma}). This motivates the so-called universal hypothesis testing problem, originally proposed by \citet{Hoeffding1965}: {\emph{does there exist a nonparametric goodness-of-fit test that achieves the same optimal exponential decay rate as in the simple hypothesis testing problem where $Q$ is known?}}  Over the years, universally optimal tests only exist when the sample space is finite, i.e., when $P$ and $Q$ are both multinomial \cite{Hoeffding1965,Unnikrishnan2011UHT}. For a more general sample space, attempts have been largely fruitless with the only exception of \cite{Zeitouni1991universal,Yang2017RobustKLD}. Their results, however, were obtained at the cost of a weaker optimality and the proposed tests are rather complicated due to use of L\'evy-Prokhorov metric. We remark that even the existence of such a test remains unknown when the sample space is non-finite.

{\bf Contributions.} We first show a simple kernel test, comparing the MMD between the target distribution and the sample empirical distribution with a proper threshold, as an optimal approach to the universal hypothesis testing problem when the sample space is Polish, locally compact Hausdorff, e.g., $\mathbb R^d$. To the best of our knowledge, this is the first result on the universal optimality for a general, non-finite sample space. Taking into account the difficulty of obtaining closed-form integrals for non-Gaussian distributions, we then follow \cite{Lloyd2015ModelCrit} to cast the original problem into a two-sample problem. We establish the same optimality for the quadratic-time kernel two-sample tests proposed in \cite{Gretton2012}, provided that $\omega(n)$ independent samples are drawn from $P$. For the KSD based tests, the constant level constraint on the type-I error probability is difficult to satisfy for all possible sample sizes. By relaxing the constraint to an asymptotic one and assuming additional conditions, we establish the optimal exponential decay rate of the type-II error probability for the quadratic-time KSD based tests proposed in \cite{Chwialkowski2016Goodness,Liu2016GoodnessFit}. 

As another contribution, we proceed to investigate the quadratic-time kernel two-sample tests  in a more general setting where the sample sizes scale in the same order, e.g., when the two sets of samples have the same size. We show that the type-II error probability also vanishes exponentially fast. The obtained exponential decay rate is further shown to be optimal among all two-sample tests under the same level constraint, and is independent of particular kernels  provided that they are bounded continuous and characteristic.

Key to our approach are Sanov's theorem from large deviation theory \cite{Dembo2009} and the weak metrizable properties of the MMD \cite{SimSch16Kernel,Sriperumbudur2016EstPM} and the KSD \cite{Gorham2017measuringKSD}, which enable us to directly investigate the acceptance region defined by the test, rather than using the test statistic as an intermediate. 

{\bf Paper Outline.}  Section~\ref{sec:problem} introduces the asymptotic statistical criterion used in this paper and formally states the problem of universal hypothesis testing. Section~\ref{sec:relatedworks} reviews related works.  In Section~\ref{sec:mainresults}, we present two classes of MMD based tests that are optimal for universal hypothesis testing and discuss their implications to goodness of fit testing. Section~\ref{sec:KSD} considers the KSD based goodness-of-fit tests and Section~\ref{sec:two}  establishes the universal optimality of the quadratic-time MMD based two-sample tests in a more general setting. We conclude this paper in Section~\ref{sec:conclusion}.
\section{Problem}
\label{sec:problem}
Throughout this paper, let $\mathcal X$ be a Polish space (i.e., a separable completely metrizable topological space) and $\mathcal P$ the set of Borel probability measures defined on $\mathcal X$. Given a distribution $P\in\mathcal P$ and sample $x^n$ from an unknown distribution $Q\in\mathcal P$, we want to determine whether to accept $H_0:P=Q$ or $H_1:P\neq Q$. A test $\Omega(n)=\{\Omega_0(n),\Omega_1(n)\}$ partitions $\mathcal X^{n}$ into two disjoint sets with $\Omega_0(n)\cup\Omega_1(n)=\mathcal{X}^{n}$. If $x^n\in\Omega_i(n),i=0,1$, a decision is made in favor of hypothesis $H_i$. We say that $\Omega_0(n)$ is an acceptance region for the null hypothesis $H_0$ and $\Omega_1(n)$ the rejection region. A type-I error is made when $P=Q$ is rejected while $H_0$ is true, and a type-II error occurs when $P=Q$ is accepted despite $H_1$ being true. The two error probabilities are $P(\Omega_1(n)):=\mathbf P_{x^n\sim P}\left(x^n\in\Omega_1(n)\right)$ and $Q(\Omega_0(n)):=\mathbf P_{x^n\sim Q}\left(x^n\in\Omega_0(n)\right)$ with $Q\neq P$, respectively. 

In general, the two error probabilities can not be minimized simultaneously. A commonly used approach, the so-called Neyman-Pearson approach \cite{Casella2002}, is to set an upper bound $\alpha$ on the type-I error probability and considers only level $\alpha$ tests, i.e., tests with $P(\Omega_1(n))\leq\alpha$. However, similar to the two-sample problem \cite{Gretton2012}, it is not possible to distinguish distributions with high probability at a given, fixed sample, without prior assumptions on the difference between $P$ and $Q$. We therefore consider an asymptotic statistical criterion as the performance metric. 

A level $\alpha$ test is said to be consistent if the type-II error probability vanishes in the large sample limit. Such a test is exponentially consistent when the error probability additionally vanishes exponentially fast w.r.t.~the sample size, that is, when \[\liminf_{n\to\infty}-\frac{1}{n}\log Q(\Omega_0(n))>0.\]
The above limit is also referred to as the type-II error exponent in information theory. Clearly, the larger the error exponent, the faster the error probability decreases in the  sample limit. Under this criterion, an optimal test would achieve the maximum type-II error exponent while satisfying the level constraint. Error exponent is a widely used metric in source coding and channel coding \cite{Cover2006}, and is closely related to two other asymptotic statistical criteria \cite{Serfling}. In particular, the Chernoff index equals the minimum of the type-I and type-II error exponents, and the exact Bahadur slope is equivalent to twice of the type-I error exponent with a constant constraint on the type-II error probability.

We present a useful lemma which gives the optimal type-II error exponent of any level $\alpha$ test for simple hypothesis testing between two known distributions. Let $D(P\|Q)$ denote the KLD between $P$ and $Q$. That is, $D(P\|Q)=\mathbf E_P\log(dP/dQ)$ where $dP/dQ$ stands for the Radon-Nikodym derivative of $P$ w.r.t.~$Q$ when it exists, and  $D(P\|Q)=\infty$ otherwise \cite{Dembo2009}.

\begin{lemma}[Chernoff-Stein Lemma {\cite{Cover2006, Dembo2009}}]
	\label{lem:SteinLemma}
	Let $x^n$ i.i.d.~$\sim R$. Consider simple hypothesis testing between $H_0: R = P\in\mathcal P$ and $H_1:R=Q\in\mathcal P$,
	with $0<D(P\|Q)<\infty$. Given $0<\alpha<1$, let $\Omega^*(n,P,Q)=\{\Omega_0^*(n,P,Q), \Omega_1^*(n, P, Q)\}$ be the optimal level $\alpha$ test with which the type-II error probability is minimized for each $n$. It follows that
	\[\lim_{n\to\infty}-\frac{1}{n}\log Q(\Omega_0^*(n, P,Q)) = D(P\|Q).\]
\end{lemma}

{\bf Problem Statement.} Let $\Omega(n)=\{\Omega_0(n), \Omega_1(n)\}$ be a nonparametric goodness-of-fit test of level $\alpha$. With $x^n$ i.i.d.~$\sim Q$ under the alternative hypothesis, the corresponding type-II error probability $Q(\Omega_0(n))$ can not be lower than $Q(\Omega_0^*(n,P,Q))$. As such, Chernoff-Stein lemma indicates that its type-II error exponent is bounded by $D(P\|Q)$. For any given $P$, the problem is to find a goodness-of-fit test $\Omega(n)$, if it exists, so that 
{\it
	\begin{enumerate}
		\item under $H_0: P=Q$,~$\mathbf P_{x^n}(\Omega_1(n))\leq\alpha$,
		\item under $H_1: P\neq Q$,  \[\begin{aligned}
		\liminf_{n\to\infty}-\frac{1}{n}\log\mathbf P_{x^n}(\Omega_0(n)) = D(P\|Q)
		\end{aligned},\]
		for arbitrary $Q$ with $0<D(P\|Q)<\infty$,
	\end{enumerate}
}
giving rise to the name {\it universal} hypothesis testing. 
\section{Related Work} 
\label{sec:relatedworks}
The decay rate of the type-II error probability has been widely investigated for existing kernel based tests. For the simple kernel tests in \cite{Altun2006,Szabo2015Two, Szabo2016learning} and the kernel two-sample tests in \cite{Chwialkowski2015fast,Fukumizu2009,Gretton2009, Gretton2012OptKernelLarge,Sutherland2017GeneandCrit,Zaremba2013Btest}, analysis is based on the test statistics, through their asymptotic distributions or some probabilistic bounds on their convergence to the population statistics. The resulting characterizations depend on kernels and are loose in general. For the KSD based tests, current statistical characterization is limited to consistency; the asymptotic distributions of the test statistics either have no closed form \cite{Chwialkowski2016Goodness} or are hard to analyze \cite{Jitkrittum2017linearGoodness,Liu2016GoodnessFit}.

Other asymptotic statistical criteria have also been used for comparing nonparametric goodness-of-fit tests.  \citet{Jitkrittum2017linearGoodness} used the approximate Bahadur slope and showed that their linear-time test has greater relative efficiency than the linear-time test proposed in  \cite{Liu2016GoodnessFit}, assuming a mean-shift alternative. However, it is not clear whether such a result holds for a more general alternative. \citet{BalaMinimaxOptGOF} investigated the detection boundary and showed that the simple kernel test is suboptimal under this criterion. A minimax optimal test was then proposed for a composite alternative, where the worst-case performance w.r.t.~a set of probability measures is optimized. In contrast, our optimality criterion is much stronger in that the optimality must hold for any distribution defining the alternative hypothesis; specifically, the nonparametric test must achieve the maximum type-II error exponent $D(P\|Q)$ for any $Q$ satisfying $0<D(P\|Q)<\infty$.

\section{Maximum Mean Discrepancy  Based Goodness-of-Fit Tests}
\label{sec:mainresults}
This section studies two classes of MMD based tests for universal hypothesis testing, followed by discussions on related aspects. We begin with a brief review of the MMD and of Sanov's theorem.

Let $\mathcal H_k$ be an RKHS defined on $\mathcal X$ with reproducing kernel $k$.  The mean embedding of $P\in\mathcal P$ in $\mathcal H_k$ is a unique element $\mu_{k}(P)\in\mathcal H_k$ such that $\mathbf{E}_{y\sim P}f(y)=\langle f, \mu_k(P)\rangle_{\mathcal H_k}$ for all $f\in\mathcal H_k$ \cite{Berlinet2011RKHS}. We assume that $k$ is  bounded continuous, hence the existence of $\mu_k(P)$ is guaranteed by the Riesz representation theorem. The MMD between two probability measures $P$ and $Q$ is defined as the RKHS-distance between their mean embeddings, which can be expressed as
\begin{align}
&~d_k(P,Q)\nn\\=&~\|\mu_k(P)-\mu_k(Q)\|_{\mathcal H_k}\nn\\=&~\left(\mathbf E_{yy'}k(y,y')+\mathbf{E}_{xx'}k(x,x')-2\mathbf E_{yx}k(y,x)\right)^{1/2},\nn
\end{align}
where $y,y'$ i.i.d.~$\sim P$ and $x,x'$ i.i.d.~$\sim Q$. 

If the mean embedding $\mu_k$ is an injective map, then the kernel $k$ is said to be characteristic and the MMD $d_k$ becomes a metric on $\mathcal{P}$ \cite{Sriperumbudur2010hilbert}. A weak metrizable property of $d_k$ has also been established recently. Consider the weak topology on $\mathcal P$ induced by the weak convergence: a sequence of probability measures $P_l\to P$ weakly if and only if $\mathbf E_{y\sim P_l} f(y)\to\mathbf E_{y\sim P} f(y)$ for every bounded continuous function $f:\mathcal X\to\mathbb R$. The following theorem states when $d_k$ metrizes this weak convergence.\footnote{Indeed, \citet{SimSch16Kernel} show that $\mathcal X$ only needs to be locally compact Hausdorff. We require $\mathcal X$ be Polish in order to utilize Sanov's theorem.}
\begin{theorem}[{\cite[Theorem 55]{SimSch16Kernel}, \cite[Theorem~3.2]{Sriperumbudur2016EstPM}}]
	\label{thm:MMDmetrize}
	If $\mathcal X$ is Polish, locally compact Hausdorff, and $k$ is continuous and characteristic, then $d_k$ metrizes the weak convergence on $\mathcal P$.
\end{theorem}	

We note that the weak metrizable property is also favored for training deep generative models \cite{ArjovskyB17, Arjovsky2017WGAN,Li2017mmdGAN}. An example of Polish, locally compact Hausdorff space is $\mathbb R^d$, and both Gaussian and Laplacian kernels defined on it are bounded continuous and characteristic \cite{Sriperumbudur2016EstPM}. 

We next introduce Sanov's theorem from large deviation theory, which, together with the weak metrizable property of the MMD, is critical to establish our main results in this section. Denote by $\hat Q_n$  the empirical measure of $x^n$, i.e., $\hat Q_n=\frac{1}{n}\sum_{i=1}^n\delta_{x_i}$ with $\delta_x$ being the Dirac measure at~$x$. 

\begin{theorem}[Sanov's Theorem {\cite{Sanov1958original,Dembo2009}}]
	\label{thm:Sanov}
	Let $x^n$ i.i.d.~$\sim Q\in\mathcal P$. For a set $\Gamma\subset\mathcal P$, it holds that
	\begin{align}
	\limsup_{n\to\infty}-\frac{1}{n}\log\mathbf P_{x^n}(\hat Q_n\in\Gamma)&\leq\inf_{R\in\operatorname{int}\Gamma}D(R\|Q),\nn\\
	\liminf_{n\to\infty}-\frac{1}{n}\log\mathbf P_{x^n}(\hat Q_n\in\Gamma)&\geq\inf_{R\in\operatorname{cl}\Gamma}D(R\|Q),\nn
	\end{align}
	where $\operatorname{int}\Gamma$ and $\operatorname{cl}\Gamma$ are the  interior and closure of $\Gamma$ w.r.t.~the weak topology on $\mathcal P$, respectively.
\end{theorem}

Sanov's theorem states that if the underlying distribution $Q$ is not in $\operatorname{cl}\Gamma$, the closure of a set $\Gamma$ of distributions, then the probability of its empirical distribution $\hat Q_n$ lying in $\operatorname{cl}\Gamma$ goes to $0$ at least exponentially fast. This enables us to directly investigate type-II error exponent through the empirical distribution and the acceptance region, rather than through the limiting performance of the test statistics. Moreover, the lower bound on the error exponent would establish the universal optimality if it is no lower than $D(P\|Q)$ for a goodness-of-fit test.

We now state the two classes of MMD based goodness-of-fit tests that are universally optimal.
\subsection{Simple Kernel Tests}
\label{sec:simple}

The first test directly computes the MMD between the target distribution $P$ and the empirical distribution of sample $x^n$. Though having been investigated in  \cite{Altun2006,BalaMinimaxOptGOF,Szabo2015Two, Szabo2016learning}, its optimality for the universal hypothesis testing problem remains unknown.

Let also $\hat Q_n$ be the empirical measure of $x^n$. We have a simple kernel test with acceptance region
\[\Omega_0(n)= \left\{x^n:d_k(P, \hat Q_n)\leq\gamma_{n}\right\},\]
where $\gamma_n$ represents a threshold and  $d_k^2(P,\hat Q_n)$ equals \[\frac{1}{n^2}\sum_{i=1}^n\sum_{j=1}^nk(x_i,x_j)+\mathbf{E}_{yy'}k(y,y')
-\frac{2}{n}\sum_{i=1}^n\mathbf E_yk(x_i,y),\]
with $y,y'~\text{i.i.d.}\sim P$. On the one hand, we want the threshold $\gamma_n$ to be small so that the test type-II error probability is low; on the other hand, the threshold cannot be too small in order to meet the level constraint on the type-I error probability. The balance between the two error probabilities is attained with a  threshold that diminishes at an appropriate rate. 
\begin{theorem}
	\label{thm:simple1}
	Let $\mathcal X$ be Polish, locally compact Hausdorff. For $P\in\mathcal P$ and $x^n$ i.i.d.~$\sim Q\in\mathcal P$, assume $0<D(P\|Q)<\infty$ under the alternative hypothesis $H_1$. Further assume that kernel $k$ is bounded continuous and characteristic, with $0\leq k(\cdot,\cdot)\leq K$ for some $K>0$. For a given~$\alpha$, $0<\alpha<1$, set $\gamma_n=\sqrt{2K/n}\left(1+\sqrt{-\log\alpha}\right).$
	Then the simple kernel test $d_k(P,\hat Q_n)\leq \gamma_n$ is an optimal level $\alpha$ test for the universal hypothesis testing problem, that is, 
	\begin{enumerate}
		\item under $H_0:P=Q$,~$\begin{aligned}
		\mathbf P_{x^n}\left(d_k(P,\hat Q_n)>\gamma_n\right)\leq\alpha,
		\end{aligned}$
		\item under $H_1: P\neq Q$,\[\begin{aligned}\liminf_{n\to\infty}-\frac1n \log\mathbf P_{x^n} \left(d_k(P,\hat Q_n)\leq\gamma_n\right)=D(P\|Q).
		\end{aligned}\]
	\end{enumerate}
\end{theorem}
\begin{proof}
	That $d_k(P,\hat Q_n)\leq \gamma_n$ has level $\alpha$ can be directly verified by \cite[Eq.~(24)]{Szabo2016learning} (see Lemma~\ref{lem:gamman} in Appendix~\ref{sec:proofA}). Let $\beta=\liminf_{n\to\infty}-\frac1n\log\mathbf P_{x^n}(d_k(P,\hat Q_n)\leq\gamma_n)$ under $H_1$. According to Chernoff-Stein lemma, we only need to show $\beta\geq D(P\|Q)$.
	
	To apply Sanov's theorem, we notice that deciding if $x^n\in\{x^n:d_k(P,\hat Q_n)\leq \gamma_n\}$ is equivalent to deciding if its empirical measure $\hat Q_n\in \{P':d_k(P,P')\leq \gamma_n\}$. Since $\gamma_n\to0$ as $n\to\infty$, for any constant $\gamma>0$, there exists an integer $n_0$ such that $\gamma_n<\gamma$ for all $n>n_0$. Hence, $\{P':d_k(P,P')\leq \gamma_n\}\subset\{P':d_k(P,P')\leq\gamma\}$ for large enough $n$. It follows that for any $\gamma>0$,
	\begin{align}
	\label{eqn:common1}
	\beta &\geq\liminf_{n\to\infty}-\frac{1}{n}\log \mathbf P_{x^n}\left(d_k(P,\hat Q_n)\leq\gamma \right)
	\nn\\&\geq\inf_{\{P'\in\mathcal P:d_k(P,P')\leq\gamma\}}D(P'\|Q),
	\end{align}
	where the last inequality is from Sanov's theorem and that $\{P'\in\mathcal P:d_k(P,P')\leq\gamma\}$ is closed w.r.t.~the weak topology (cf.~Theorem~\ref{thm:MMDmetrize}). Then for any given $\epsilon>0$, there exists some $\gamma>0$ such that $\inf_{\{P'\in\mathcal P:d_k(P,P')\leq\gamma\}} D(P'\|Q)\geq D(P\|Q)-\epsilon$, using the lower semi-continuity of the KLD \cite{VanErven2014RenyiKLD} (Lemma~\ref{lem:lowerSemiContiKLD} in Appendix~\ref{sec:proofA}) and the assumption that $0<D(P\|Q)<\infty$ under $H_1$. This further implies $\beta\geq D(P\|Q).$
\end{proof}
It is worth noting that we simply select one threshold $\gamma_n$ in the above theorem. Indeed, any vanishing threshold $\gamma_n'>0$ with $\gamma_{n}'\geq\gamma_n$ leads to the same optimality w.r.t.~the type-II error exponent, an asymptotic statistical criterion. A larger threshold, however, may result in a higher type-II error probability in the finite sample regime. A further discussion on the threshold choice will be given in Section~\ref{sec:remark}. 

The test statistic $d_k^2(P,\hat Q_n)$ is a biased estimator of $d_k^2(P,Q)$. By replacing $\frac{1}{n^2}\sum_{i=1}^n\sum_{j=1}^n k(x_i,x_j)$ with $\frac{1}{n(n-1)}\sum_{i=1}^n\sum_{j\neq i}k(x_i,x_j)$, we obtain an unbiased statistic denoted as $d_u^2(P,\hat Q_n)$. We comment that $d_u^2(P, \hat Q_n)$ is not a squared quantity and can be negative, due to the unbiasedness. The following result shows that $d_u^2(P,\hat Q)$ can also be used to construct a universally optimal test.

\begin{corollary}
	\label{cor:simple2}
	Under the same conditions of Theorem~\ref{thm:simple1}, the test $d_u^2(P,\hat Q_n)\leq\gamma_n^2+K/n$ is a level~$\alpha$ asymptotically optimal test for universal hypothesis testing.
\end{corollary}
\begin{proof}[Proof (sketch)]
	As $0\leq k(\cdot,\cdot)\leq K$, we get $\{x^n:d_k^2(P,\hat Q_n)\leq\gamma_n^2\}\subset\{x^n:d_u^2(P,\hat Q_n)\leq\gamma_n^2+K/n\}\subset\{x^n:d_k^2(P,\hat Q_n)\leq\gamma_n^2+2K/n\}$.  The level constraint and the type-II error exponent can then be verified using the subset and superset, respectively. See Appendix~\ref{sec:proofA} for details.
\end{proof}

The tests in this section still require closed-form integrals, namely, $\mathbf E_yk(x_i,y)$ and $\mathbf E_{yy'}k(y,y')$. Our purpose here is to show that the universally optimal type-II error exponent is indeed achievable, giving a meaningful optimality criterion for goodness-of-fit tests. In the next section, we consider another class of MMD based tests without the need of closed-form integrals.

\subsection{Kernel Two-Sample Tests}
\label{sec:twosamplemainresult}
In the context of model criticism, \citet{Lloyd2015ModelCrit} cast goodness of fit testing into a two-sample problem, where one draws sample $y^m$ from distribution $P$ and then decide if $y^m$ and $x^n$ are from the same distribution. A question that arises is the choice of number of samples, which is not obvious due to the lack of an explicit criterion. In light of universal hypothesis testing, we could ask how many samples would suffice for a two-sample test to attain the error exponent $D(P\|Q)$.

Denote by $\hat P_m$ the empirical measure of $y^m$. Notice that the type-I and type-II error probabilities of a two-sample test depend on both $P$ and $Q$. We consider the following two-sample test with acceptance region \[\Omega_0(m,n)=\{(y^m,x^n):d_k(\hat P_m,\hat Q_n)\leq\gamma_{m,n}\},\]
where $K$ is a finite bound on $k(\cdot,\cdot)$,
\begin{align} \gamma_{m,n}=\left(\sqrt{{K}/{m}}+\sqrt{{K}/{n}}\right)\left(2+\sqrt{-2\log(\alpha/2)}\right),\nn
\end{align}
\begin{align}
d_k^2(\hat P_m,\hat Q_n)=&\,\sum_{i=1}^n\sum_{j=1}^nk(x_i,x_j)+\sum_{i=1}^m\sum_{j=1}^mk(y_i,y_j)\nn\\
&\,-\frac{2}{mn}\sum_{i=1}^n\sum_{j=1}^mk(x_i,y_j).\nn
\end{align}

The statistic $d_k^2(\hat P_m, \hat Q_n)$ for estimating the  squared MMD was originally proposed in \cite{Gretton2012}. Although additional randomness is introduced due to the use of $\hat P_m$, it does not hurt the type-II error exponent provided that $m$ is large enough, as stated below.
\begin{theorem} 
	\label{thm:2sample1}
	Assume the same conditions as in Theorem~\ref{thm:simple1}, and that $y^m$ i.i.d.~$\sim P$ and $x^n$ i.i.d.~$\sim Q$.  Let $\Omega_1(m,n)=\mathcal X^{m+n}\setminus\Omega_0(m,n)$ be the rejection region. If $m$ is such that $m/n\to\infty$ as $n\to\infty$, then we have
	\begin{enumerate}	
		\item under $H_0: P=Q$, $\begin{aligned}\mathbf P_{y^mx^n}\left(\Omega_1(m,n)\right)\leq\alpha\end{aligned}$,
		\item under $H_1: P\neq Q$, 
		\begin{align}
		\label{eqn:diff1}
		\liminf_{n\to\infty}-\frac{1}{n}\log\mathbf P_{y^mx^n}\left(\Omega_0(m,n)\right)=D(P\|Q).\end{align}		
	\end{enumerate}
\end{theorem}

 The level $\alpha$ constraint can be verified by \cite[Theorem~7]{Gretton2012}. We decompose the type-II error probability into two components and show that each decays at least exponentially at a rate of $D(P\|Q)$. A complete proof is provided in Appendix~\ref{sec:proofB}.

We may also replace the first two terms in $d_k^2(\hat P_m,\hat Q_n)$ with $\frac{1}{n(n-1)}\sum_{i=1}^n\sum_{j\neq i} k(x_i,x_j)$ and $\frac{1}{m(m-1)}\sum_{i=1}^m\sum_{j\neq i} k(y_i,y_j)$, which results in an unbiased statistic denoted as $d_u^2(\hat P_m, \hat Q_n)$ \cite{Gretton2012}. The following corollary can be shown in a similar manner to Corollary~\ref{cor:simple2} by noting that $|d_u^2(\hat P_m,\hat Q_n)-d_k^2(\hat P_m,\hat Q_n)|\leq K/m+K/n$; details are omitted.
\begin{corollary}
	\label{cor:2sample2}
	Under the same assumptions of Theorem~\ref{thm:2sample1}, the test $d_u^2(\hat P_m,\hat Q_n)\leq\gamma_{m,n}^2+K/m+K/n$ has its type-I error probability below $\alpha$ and type-II error exponent being $D(P\|Q)$, when ${m}/{n}\to\infty$ as $n\to\infty$.
\end{corollary}

\subsection{Remarks}
\label{sec:remark}
{\bf Threshold Choice.} The distribution-free thresholds used in the MMD based tests are generally too conservative, as the actual distribution $P$ is not taken into account. Alternatively, we may use Monte Carlo or bootstrap methods to empirically estimate the acceptance threshold \cite{Chwialkowski2016Goodness,Gretton2012,Jitkrittum2017linearGoodness}, making the tests asymptotically level $\alpha$. These methods, however, introduce additional randomness on the threshold choice and further on the type-II error probability. As a result, it becomes difficult to characterize the type-II error exponent. A simple fix is to take the minimum of the Monte Carlo or bootstrap threshold and the distribution-free one, guaranteeing a vanishing threshold and hence the optimal type-II error exponent.  In our experiments, the bootstrap threshold is always smaller than the distribution-free threshold.

{\bf Finite vs.~Asymptotic Regimes.} A finitely positive error exponent $D(P\|Q)$ implies that the error probability decays with $\mathcal O\left(2^{-n(D(P\|Q)-\epsilon)}\right)$ where $\epsilon\in(0,D(P\|Q))$ can be arbitrarily small. It further implies that kernels affect only the sub-exponential term in the type-II error probability, as long as they are bounded continuous and characteristic. When $n$ is small, the sub-exponential term may dominate and the test performance does depend  on the specific kernel. Selecting a proper kernel is an ongoing research topic and we refer the reader to related works such as \cite{Jitkrittum2017linearGoodness,Gretton2012OptKernelLarge,Sutherland2017GeneandCrit}.

{\bf Non-i.i.d.~Sample.} We notice that \citet{Chwialkowski2016Goodness} considered non-i.i.d. sample by use of wild bootstrap. In general, statistical optimality with non-i.i.d.~sample is difficult to establish even for simple hypothesis testing.

{\bf General Two-Sample Problem.} Studied in Section~\ref{sec:twosamplemainresult} can be seen as a special case of the two-sample problem where sample sizes scale in different orders, i.e., $m/n\to\infty$ as $n\to\infty$. A direct extension is to consider the more common setting where $0<\lim_{n\to\infty}m/n<\infty$. For example, an equal number of real and fake samples is typically used for training generative models where the MMD acts as a critic to distinguish between them \cite{Li2015Generative,Dziugaite2015TrainingGenerative,Li2017mmdGAN}. However, the current approach is not readily applicable, for lacking an extended version of Sanov's theorem that works with two sample sequences. A naive way may try decomposing the acceptance region $\Omega_0(m,n)$ into $\Omega_0'(m)\times\Omega_0''(n)$ with $\Omega_0'(m)$ and $\Omega_0''(n)$ being respectively decided by $y^m$ and $x^n$, and then apply Sanov's theorem to each set. Unfortunately, such a  decomposition is not possible for the MMD based two-sample tests. We postpone a further investigation until Section~\ref{sec:two}, after studying the KSD based goodness-of-fit tests in the next section.


\section{Kernel Stein Discrepancy Based Goodness-of-Fit Tests}
\label{sec:KSD}
In this section, we investigate the KSD based goodness-of-fit tests recently proposed in \cite{Chwialkowski2016Goodness,Jitkrittum2017linearGoodness,Liu2016GoodnessFit}. 

Let $\mathcal X=\mathbb R^d$. Denote by $p$ and $q$ the density functions (w.r.t.~Lebesgue measure) of $P$ and $Q$, respectively. In \cite{Chwialkowski2016Goodness, Liu2016GoodnessFit}, the KSD is defined as
\begin{align}
d_S(P,Q)=\max_{\|f\|_{\mathcal H_k}\leq 1} \mathbf E_{x\sim Q}\left[s_p(x)f(x)+\nabla_xf(x)\right],\nn
\end{align}
where $\|f\|_{\mathcal H_k}\leq 1$ denotes the unit ball in the RKHS $\mathcal H_k$, and $s_p(x)=\nabla_x\log p(x)$ is the score function of $p(x)$. An equivalent expression of the KSD is given by \[d_S^2(P,Q)=\mathbf E_{x\sim Q}\mathbf E_{x'\sim Q}\, h_p(x,x'),\]
where $h_p(x,y)=s_p^T(x)s_p(y)k(x,y)+s_p^T(y)\nabla_xk(x,y)+s_p^T(x)\nabla_yk(x,y)+\operatorname{trace}(\nabla_{x,y}k(x,y))$. Given sample $x^n$, we may estimate $d_S^2(P, Q)$ by $d_S^2(P,\hat Q_n)=\frac{1}{n^{2}}\sum_{i=1}^n\sum_{j=1}^n h_p(x_i,x_j)$, which is a degenerate V-statistic under the null hypothesis $H_0:P=Q$ \cite{Chwialkowski2016Goodness}.

With $\mathbf E_{x\sim Q}\|\nabla_x\log p(x)-\nabla_x\log q(x)\|^2\leq\infty$ and a $C_0$-universal kernel \cite{Carmeli2010vector}, $d_S(P,Q)=0$ if and only if $P=Q$ \cite[Theorem~2.2]{Chwialkowski2016Goodness}. A nice property of the KSD is that this result requires only the knowledge of $p(x)$ up to the normalization constant. The KSD has also been shown to be lower bounded in terms of the MMD or the bounded Lipschitz metric (involving some unknown constants) under suitable conditions \cite{Gorham2017measuringKSD}. This indicates that $d_S(P,P_l)\to 0$ only if $P_l\to P$ weakly, which is important to applying Sanov's theorem in our approach.

Unlike the MMD based test statistics, there does not exist a uniform or distribution-free probabilistic bound on $d_S^2(P,\hat Q_n)$. As a result, it is difficult to find a test threshold to meet the fixed level constraint for all sample sizes. To proceed, we relax the level constraint to an asymptotic one, 
and use the result of \cite[Proposition~3.2]{Chwialkowski2016Goodness} which shows that $nd_S^2(P,\hat Q_n)$ converges weakly to some distribution under $H_0$.\footnote{\citet{Chwialkowski2016Goodness} assume $\tau$-mixing as the notion of dependence within the observations, which holds in the i.i.d.~case. They also assume a technical condition $\sum_{t=1}^\infty t^2\sqrt{\tau(t)}\leq\infty$ on $\tau$-mixing. See details in \cite{Chwialkowski2016Goodness, Dedecker2007}.} We assume a fixed $\alpha$-quantile $\gamma_{\alpha}$ of the limiting cumulative distribution function, so that $\lim_{n\to\infty}P(d_S^2(P,\hat Q_n)>\gamma_\alpha/n)=\alpha$. Then if $\gamma_n$ is such that $\gamma_n\to0$ and $\lim_{n\to\infty}n\gamma_n\to\infty$, e.g., $\gamma_n=\sqrt{1/n}\left(1+\sqrt{-\log\alpha}\right)$, we get $\gamma_n>\gamma_\alpha/n$ in the limit and thus $\lim_{n\to\infty}P(d_S^2(P,\hat Q_n)>\gamma_n)\leq\alpha$. Similarly, this threshold choice may be poor in the finite sample regime and we can take the minimum of this threshold and a bootstrap one \cite{Arcones1992bootstrap,Chwialkowski2014wild,Leucht2012degenerate}. Together with the weak convergence properties of the KSD, we have the following result. 
\begin{theorem}
	\label{thm:KSDmain}
	Let $P$ and $Q$ be distributions defined on $\mathbb R^d$, with $0<D(P\|Q)<\infty$ under the alternative hypothesis. Assume $x^n~\text{i.i.d.}\sim Q$ and set $\gamma_n=\sqrt{1/n}\left(1+\sqrt{-\log\alpha}\right)$. It follows that
	\begin{enumerate}
		\item if $h_p$ is Lipschitz continuous and $\mathbf E_{x\sim Q} h_p(x,x)<\infty$, then {under} $H_0:P=Q$,
		\[\lim_{n\to\infty} \mathbf P_{x^n}\left(d_S^2(P,\hat Q_n)> \gamma_n\right)\leq\alpha.\] 	
		\item if 1) $d=1$, $k(x,y)=\Phi(x-y)$ for some $\Phi\in{C^2}$ (twice continuous differentiable) with a non-vanishing generalized Fourier transform; 2) $k(x,y)=\Phi(x-y)$ for some $\Phi\in{C^2}$ with a non-vanishing generalized Fourier transform, and the sequence $\{\hat Q_n\}_{n\geq1}$ is uniformly tight; 3) $k(x,y)=(c^2+\|x-y\|_2^2)^\eta$ for $c>0$ and $-1<\eta<0$, then {under} $H_1:P\neq Q$,
		\[\liminf_{n\to\infty}-\frac1n\log \mathbf P_{x^n}\left(d_S^2(P,\hat Q_n)\leq \gamma_n\right)=D(P\|Q).\]
	\end{enumerate}
\end{theorem}
\begin{proof}[Proof (sketch)] 
	The condition for the asymptotic level constraint is taken from \cite[Proposition~3.2]{Chwialkowski2016Goodness}. To establish the type-II error exponent, let $d_{W}$ denote the MMD or the bounded Lipschitz metric,  which  metrize the weak convergence on $\mathcal P$. Under each of the three conditions from \cite[Theorems 5, 7, and 8]{Gorham2017measuringKSD},  $d_{W}(P,\hat Q_n)\leq g(d_{S}(P,\hat Q_n))$ where $g(d_{S})\to 0$ as $d_{S}\to 0$. Then there exists $\gamma_n'$ such that $\{x^n:d^2_{S}(P,\hat Q_n)\leq\gamma_n\}\subset\{x^n:d^2_{W}(P,\hat Q_n)\leq\gamma_n'\}$ and $\gamma_n'\to 0$ as $n\to\infty$. Thus, the type-II error exponent is lower bounded by $D(P\|Q)$,  following the same  argument of Theorem~\ref{thm:simple1}. The upper bound is from Chernoff-Stein lemma which also holds for an asymptotic level constraint.
\end{proof}

\citet{Liu2016GoodnessFit} proposed an unbiased U-statistic $d_{S(u)}^2(P,\hat Q_n)=\frac{1}{n(n-1)}\sum_{i=1}^n\sum_{j\neq i} h_p(x_i,x_j)$ for estimating $d_S^2(P,Q)$. A similar result holds under an additional assumption on the boundedness of $h_p(\cdot,\cdot)$, using the same argument of Corollary~\ref{cor:simple2}.
\begin{corollary}
	Assume the same conditions as in Theorem~\ref{thm:KSDmain} and further that $h_p(\cdot,\cdot)\leq H_p$ for some $H_p\in\mathbb R^+$. Then the test $d_{S(u)}^2(P,\hat Q_n)\leq \gamma_n+H_p/n$ is asymptotically level $\alpha$ and achieves the optimal type-II error exponent $D(P\|Q)$.
\end{corollary}
{\bf The Weak Convergence Property.} To use Sanov's theorem, we find a superset of probability measures for the equivalent acceptance region, which is required to be closed and to converge (in terms of weak convergence) to $P$ in the large sample limit. Without the weak convergence property, the equivalent acceptance region may contain probability measures that are not close to $P$, and the minimum KLD over the superset would be hard to obtain. An example can be found in \cite[Theorem~6]{Gorham2017measuringKSD}  where the KSDs are driven to zero by sequences of probability measures not converging to $P$. Consequently, our approach does not establish the optimal type-II error exponent for the linear-time KSD based tests in \cite{Jitkrittum2017linearGoodness, Liu2016GoodnessFit}, the linear-time kernel two-sample test in \cite{Gretton2012}, the B-test in \cite{Zaremba2013Btest}, and a pseudometric based two-sample test in \cite{Chwialkowski2015fast}, due to lack of the weak convergence property. 


\section{General Two-Sample Problem}
\label{sec:two}
In this section, we investigate the kernel two-sample tests in a more general setting. As discussed in Section~\ref{sec:remark}, the key is to establish an extended Sanov's theorem that is able to handle two sample sequences. 
\subsection{Extended Sanov's Theorem}
We define pairwise weak convergence for probability measures: we say $(P_l,Q_l)\to(P,Q)$ weakly if and only if both $P_l\to P$ and $Q_l\to Q$ weakly. We consider $\mathcal P\times\mathcal P$ endowed with the topology induced by this pairwise weak convergence.  It can be verified that this topology is equivalent to the product topology on $\mathcal P\times\mathcal P$ where each $\mathcal P$ is endowed with the topology of weak convergence. An extended version of Sanov's theorem is stated below. 

\begin{theorem}[Extended Sanov's Theorem] 
	Let $\mathcal X$ be a Polish space, $y^m$~i.i.d.~$\sim P$, and $x^n$~i.i.d.~$\sim Q$. Assume $0<\lim_{m,n\to\infty}\frac{m}{m+n}=c<1$. Then for a set $\Gamma\subset\mathcal P\times\mathcal P$, it holds that 
	\begin{align}
	&~\inf_{(R,S)\in\operatorname{int}\Gamma} cD(R\|P)+(1-c)D(S\|Q)\nn\\
	\geq&~\limsup_{m,n\to\infty}-\frac{1}{m+n}\log\mathbf {P}_{y^mx^n}((\hat{P}_m,\hat{Q}_n)\in\Gamma)\nn\\
	\geq
	&~\liminf_{m,n\to\infty}-\frac{1}{m+n}\log\mathbf{P}_{y^mx^n} ((\hat{P}_m,\hat{Q}_n)\in\Gamma)\nn\\
	\geq &~\inf_{(R,S)\in\operatorname{cl}\Gamma} cD(R\|P)+(1-c)D(S\|Q),\nn
	\end{align}
	where $\operatorname{int}\Gamma$ and $\operatorname{cl}\Gamma$ denote the interior and closure of $\Gamma$ w.r.t.~the pairwise weak convergence, respectively.
\end{theorem}	
We comment that this extension is not apparent as existing tools, e.g., Cram{\' e}r theorem \cite{Dembo2009}, used for proving Sanov's theorem can only deal with a single distribution. In Appendix~\ref{sec:extendedSanov}, we first prove the above result in finite sample space and then extend it to general Polish space, with two simple combinatorial lemmas as prerequisites. 

\subsection{Exact and Optimal Error Exponent}
\label{sec:ExpConsistency}
With the extended Sanov's theorem and a vanishing threshold $\gamma_{m,n}$, we are ready to establish the exponential decay of the type-II error probability. A proof is provided in Appendix~\ref{sec:proofmainresult}. 
\begin{theorem}
	\label{thm:mainresult1}
	Assume the same conditions as in Theorem~\ref{thm:2sample1}, and  $\lim_{m,n\to\infty}\frac{m}{m+n}=c\in(0,1)$. Under the alternative hypothesis $H_1:P\neq Q$, further~assume~that \[0<D^*:=\inf_{R\in\mathcal P} cD(R\|P) + (1-c)D(R\|Q)<\infty.\] Given $0<\alpha<1$, the test
	$d_k(\hat P_m, \hat Q_n)\leq \gamma_{m,n}$
	with $\gamma_{m,n}$ defined in Section~\ref{sec:twosamplemainresult} is level $\alpha$ and also exponentially consistent with the type-II error exponent being
	\[\liminf_{m,n\to\infty}-\frac1{m+n}\log \mathbf P_{y^mx^n}(\Omega_0(m,n))=D^*.\]
\end{theorem}

Here we consider the error exponent w.r.t.~$m+n$, the total number of observations for testing. Therefore, when $0<c<1$, the type-II error probability  vanishes as $\mathcal O(2^{-(m+n)(D^*-\epsilon)})$, where $\epsilon\in(0,D^*)$ is fixed and can be arbitrarily small. Similarly, this result only requires kernels be bounded continuous and characteristic.

Our next theorem provides an upper bound on the type-II error exponent of any (asymptotically) level $\alpha$ two-sample test.  This further shows that the kernel test $d_k(\hat P_m,\hat Q_n)\leq \gamma_{m,n}$ is asymptotically optimal, by choosing the type-II error exponent as the performance metric. See Appendix~\ref{sec:optimality} for a proof.

\begin{theorem}
	\label{thm:upperbound}
Assume the same conditions as in Theorem~\ref{thm:mainresult1}.  For a nonparametric two-sample test $\Omega'(m,n)=\{\Omega_0'(m,n),\Omega_1'(m,n)\}$ which is (asymptotically) level $\alpha,0<\alpha<1$, its type-II error exponent is bounded by $D^*$, that is,  \begin{align}\liminf_{m,n\to\infty}-\frac{1}{m+n}\log\mathbf P_{y^mx^n}(\Omega'_0(m,n))\leq D^*.\nn
\end{align}
\end{theorem}	
We can use Theorems~\ref{thm:mainresult1} and \ref{thm:upperbound} to identify more asymptotically optimal two-sample tests:
\begin{itemize}
	\item Assuming $n=m$, the unbiased test $d_u^2(\hat P_m, \hat Q_n)\leq(4K/\sqrt{n})\sqrt{\log(\alpha^{-1})}$, with a tighter threshold, is also level $\alpha$ \cite{Gretton2012}. As $k(\cdot,\cdot)$ is finitely bounded by $K$, its type-II error probability vanishes exponentially at a rate of $\inf_{R\in\mathcal P}\frac{1}{2}D(R\|P)+\frac12D(R\|Q)$, which can be shown by the same argument of Corollary~\ref{cor:simple2}.
	
	
	\item It is also possible to consider a family of kernels for the test statistic \cite{Fukumizu2009,Sriperumbudur2016EstPM}. For a given family $\kappa$, the test statistic is $\sup_{k\in\kappa} d_k(\hat P_m, \hat Q_n)$ which also metrizes weak convergence under suitable conditions, e.g., when $\kappa$ consists of finitely many Gaussian kernels \cite[Theorem~3.2]{Sriperumbudur2016EstPM}. If $K$ remains to be an upper bound for all $k\in\kappa$, then comparing $\sup_{k\in\kappa} d_k(\hat P_m, \hat Q_n)$ with $\gamma_{m,n}$ in Section~\ref{sec:twosamplemainresult} results in an asymptotically optimal level $\alpha$ test.
\end{itemize}
{\bf Fair Alternative.} In \cite{Ramdas2015}, a notion of fair alternative is proposed when investigating how a two-sample test performs as dimension increases. The idea is to fix $D(P\|Q)$ under the alternative hypothesis for all dimensions, guided by the fact that the KLD is a fundamental information-theoretic quantity determining the hardness of hypothesis testing problems. This approach, however, does not take into account the impact of sample sizes. In light of our results, perhaps a better choice is to fix $D^*$ defined in Theorem \ref{thm:mainresult1} when the sample sizes grow in the same order. In practice, $D^*$ may be hard to compute, so fixing its upper bound $(1-c)D(P\|Q)$ and hence $D(P\|Q)$ is reasonable.

{\bf Other Discrepancy Measures.} Other discrepancy measures between distributions may also metrize the weak convergence on $\mathcal P$, including L\'evy-Prokhorov metric, the bounded Lipschitz metric, and Wasserstein distance. We may directly compute such a discrepancy between the empirical measures and then compare it with a decreasing threshold. However, there also does not exist a uniform or distribution-free threshold such that the level constraint is satisfied for all sample sizes. A possible remedy, as in Section~\ref{sec:KSD}, is to relax  the level constraint to an asymptotic one. We will not expand into this direction, as computing such discrepancy measures from samples is generally more costly than the MMD and KSD based  statistics. 
\section{Concluding Remarks}
\label{sec:conclusion} 
In this paper, we established the statistical optimality of the MMD and KSD based goodness-of-fit tests in the spirit of universal hypothesis testing. The KSD based tests are more computationally efficient, as there is no need to draw samples or compute integrals. In comparison, the MMD based tests are statistically favorable, as they require weaker assumptions and can meet the level constraint for any sample size. The quadratic-time MMD based two-sample tests are also shown to be optimal when sample sizes scale in the same order.  Our findings not only solve a long-standing open problem in statistics, but also provide meaningful optimality criteria for  nonparametric goodness-of-fit and two-sample testing. 

While the optimality criterion is defined in the asymptotic sense, we also conduct experiments of these kernel based goodness-of-fit tests in the finite sample regime, with results given in Appendix \ref{sec:exp} due to space limit. Whereas we cannot tell much statistical difference in our experiments, some experiments in the literature showed that the MMD based tests performed better than the KSD based tests and others showed the opposite \cite{Chwialkowski2016Goodness,Gorham2017measuringKSD,Liu2016GoodnessFit,Jitkrittum2017linearGoodness}. The finite sample performance depends on kernel choice as well as specific distributions. Under the universal setting, no test is known to be optimal in terms of the type-II error probability subject to a given level constraint. Statistical optimality can only be established in the large sample limit, as the one considered in the present work.
\section*{Acknowledgement}
The authors are grateful to the anonymous reviewers for valuable comments and suggestions. The work of BC was supported in part by the U.S. National Science Foundation under grant CNS-1731237 and by the U.S. Air Force Office of Scientific Research under grant FA9550-16-1-0077. Part of this work was done when SZ and PY were students at Syracuse University.

{\small	
	\bibliography{SZhuBib}
	\bibliographystyle{abbrvnat}}
	\newpage
\onecolumn 

\begin{center}
	{\Large\bf Supplementary Material}
\end{center}
\appendix
\section{Proof of Corollary~\ref{cor:simple2}}
\label{sec:proofA}
We first present the two lemmas used in the proof of Theorem~\ref{thm:simple1}: one establishes  the convergence of $d_k(P,\hat P_m)$ and the other describes the lower semi-continuity of the KLD.
\begin{lemma}\cite{Szabo2015Two,Szabo2016learning}
	\label{lem:gamman}
	Assume $0\leq k(\cdot,\cdot)\leq K$. Given $y^m~\text{i.i.d.}\sim P$, denote by $\hat P_m$  the empirical measure of $y^m$. It follows that
	\[\mathbf{P}_{y^m}\left(d_k(P, \hat{P}_{m})> \left(2{K}/{m}\right)^{1/2}+\epsilon\right)\leq \exp{\left(-\frac{\epsilon^2m}{2K}\right)}.\]
\end{lemma}

\begin{lemma}[{\cite{VanErven2014RenyiKLD}}]
	\label{lem:lowerSemiContiKLD}
	For a fixed $Q\in\mathcal P$, $D(\cdot\|Q)$ is a lower semi-continuous function w.r.t.~the weak topology of $\mathcal P$. That is, for any $\epsilon>0$, there exists a neighborhood $U\subset\mathcal P$ of $P$ such that for any $P'\in U$, $D(P'\|Q)\geq D(P\|Q)-\epsilon$ if $D(P\|Q)<\infty$, and $D(P'\|Q)\to\infty$ as $P'$ tends to $P$ if $D(P\|Q)=\infty$.
\end{lemma}

\begin{proof}[Proof of Corollary~\ref{cor:simple2}]
	Since $0\leq k(\cdot,\cdot)\leq K$, we have 
	\begin{align}
	\left|d_u^2(P,\hat Q_n)-d^2_k(P,\hat Q_n)\right|=\left|\frac1{n^2(n-1)}\sum_{i=1}^n\sum_{j\neq i}k(x_i,x_j)-\frac1{n^2}\sum_{i=1}^nk(x_i,x_i)\right|
	\leq{K}/{n}.\nn
	\end{align} 
	It then holds that 
	\begin{align}
	\left\{x^n:d_k^2(P,\hat Q_n)\leq\gamma_n^2\right\}\subset\left\{x^n:d_u^2(P,\hat Q_n)\leq\gamma_n^2+K/n\right\}\subset\left\{x^n\hspace{-1pt}:d_k^2(P,\hat Q_n)\leq\gamma_n^2+2K/n\right\}.\nn
	\end{align}
	Thus, under $H_0: P=Q$, we have 
	\begin{align}
	P\left(d_u^2(P,\hat Q_n)>\gamma_n^2+K/n\right)\leq P\left(d_k^2(P,\hat Q_n)>\gamma_n^2\right)\leq\alpha,\nn
	\end{align}
	where the last inequality is from Lemma~\ref{lem:gamman} and the fact that $d_k(P,\hat Q_n)\geq 0$. The type-II error exponent follows from
	\begin{align}
	&~\liminf_{n\to\infty}-\frac{1}n\log Q\left(d_u^2(P,\hat Q_n)\leq\gamma_n^2+K/n\right)\nn\\
	\geq&~\liminf_{n\to\infty}-\frac{1}n\log Q\left(d_k^2(P,\hat Q_n)\leq\gamma_n^2+2K/n\right)\nn\\
	\geq&~D(P\|Q).\nn
	\end{align}
	The last inequality can be shown by similar argument of Eq.~(\ref{eqn:common1}) because $\gamma_n^2+2K/n\to 0$ as $n\to\infty$. Applying Chernoff-Stein lemma completes the proof.
	%
	%
	%
	%
	%
\end{proof}
\section{Proof of Theorem~\ref{thm:2sample1}}
\label{sec:proofB}
We use a result from \cite{Gretton2012} to verify the two-sample test to be level $\alpha$. 
\begin{lemma}[{\cite[Theorem 7]{Gretton2012}}]
	\label{lem:gammanm}	
	Let $P, Q, y^m, x^n,\hat P_m,\hat Q_n$ be defined in Theorem~\ref{thm:2sample1}. Assume $0\leq k(\cdot,\cdot)\leq K$. Then under the null hypothesis $H_0:P=Q$, 
	
	\[\mathbf{P}_{y^mx^n}\left(d_k(\hat P_m,\hat Q_n)>2(K/m)^{{1}/{2}}+2(K/n)^{{1}/{2}}+\epsilon\right)
	\leq 2\exp\left(-\frac{\epsilon^2mn}{2K(m+n)}\right).\]
\end{lemma} 
\begin{proof}[Proof of Theorem~\ref{thm:2sample1}]
	That the two-sample test is level $\alpha$ can be verified by the above lemma. The rest is to show the type-II error exponent being $D(P\|Q)$.
	
	We can write the type-II error probability as \[\mathbf P_{y^mx^n}\left(d_k(\hat P_m,\hat Q_n)\leq\gamma_{m,n}\right)=\beta_{m,n}^u+\beta_{m,n}^l,\]
	where
	\begin{align}
	\gamma_{m,n}'&=\sqrt{2K/m}+\sqrt{2KnD(P\|Q)/m}\nn\\
	\beta_{m,n}^u&=\mathbf P_{y^mx^n}\left(d_k(\hat P_m,\hat Q_n)\leq\gamma_{m,n},d_k(P,\hat P_m)>\gamma_{m,n}'\right),\nn\\
	\beta_{m,n}^l&=\mathbf P_{y^mx^n}\left(d_k(\hat P_m,\hat Q_n)\leq \gamma_{m,n}, d_k(P,\hat P_m)\leq\gamma_{m,n}'\right)\nn.
	\end{align}
	It suffices to show that $\max\{\beta^u_{m,n},\beta^l_{m,n}\}$ decreases exponentially as $n$ scales. We first have
	\begin{align}
	\label{eqn:betaupper}
	\beta_{m,n}^u\leq \mathbf P_{y^m}\left(d_k(P,\hat P_m)> \gamma_{m,n}'\right)\leq e^{-nD(P\|Q)},
	\end{align}
	where the  last inequality is due to Lemma~\ref{lem:gamman}. Thus, $\beta_{m,n}^u$ vanishes at least exponentially fast with the error exponent being $D(P\|Q)$.
	
	For $\beta_{m,n}^l$, we have 
	\begin{align}
	\beta_{m,n}^l&=\sum_{\{\hat{P}_m:d_k(P, \hat P_m)\leq\gamma_{m,n}'\}}P\left(\hat P_m\right)\,Q\left(d_k(\hat P_m, \hat Q_n)<\gamma_{m,n}\right) \nn\\
	&=\left(\sum_{\hat{P}_m:d_k(P,\hat P_m)\leq\gamma_{m,n}'}P(\hat P_m)\right)\,\sup_{\{\hat{P}_m:d_k(P,\hat P_m)\leq\gamma_{m,n}'\}}Q\left(d_k(\hat P_m, \hat Q_n)<\gamma_{m,n}\right) \nn\\
	&\leq\sup_{\{\hat{P}_m:d_k(P,\hat P_m)\leq\gamma_{m,n}'\}}Q\left(d_k(\hat P_m, \hat Q_n)<\gamma_{m,n}\right)\nn\\
	&\leq Q\left(d_k(P,\hat Q_n)\leq\gamma_{m,n}+\gamma_{m,n}'\right),\nn
	\end{align}
	where the last inequality is from the triangle inequality for metric $d_k$. Similar to Eq.~(\ref{eqn:common1}), we get
	\begin{align}
	\liminf_{n\to\infty}-\frac1n\log\beta_{m,n}^l\geq D(P\|Q),\nn
	\end{align}
	because $\gamma_{m,n}+\gamma_{m,n}'\to0$ as $n\to\infty$. Together with Eq.~(\ref{eqn:betaupper}), we have under $H_1:P\neq Q$,
	\begin{align}
	\liminf_{n\to\infty}-\frac{1}{n}\log \mathbf P_{y^mx^n}\left(d_k(\hat P_m,\hat Q_n)\leq\gamma_{m,n}\right)\geq D(P\|Q).\nn
	\end{align}
	
	We next show the other direction under $H_1$. We can write
	\begin{align}
	\mathbf P_{y^mx^n}\left(d_k(\hat P_m,\hat Q_n)\leq \gamma_{m,n}\right)
	\stackrel{(a)}{\geq}&~\mathbf P_{y^mx^n}\left(d_k(\hat P_m, P)\leq\gamma_{m}', d_k(P,\hat Q_n)\leq{\gamma_{n}'}\right)\nn\\
	=&~P\left(d_k(\hat P_m, P)\leq\gamma_{m}'\right) Q\left(d_k(P,\hat Q_n)\leq{\gamma_{n}'}\right),\nn
	\end{align}
	where $(a)$ is because $d_k$ is a metric, and we choose  $\gamma_m'=\sqrt{2K/m}\left(1+\sqrt{-\log\alpha}\right)$ and $\gamma_n'=\sqrt{2K/n}\left(1+\sqrt{-\log\alpha}\right)$ so that $\gamma_{m,n}>\gamma_m'+\gamma_n'$. Then Lemma~\ref{lem:gamman} gives $P(d_k(P,\hat P_m)\leq\gamma_m')>1-\alpha$ and $P(d_k(P,\hat Q_n)\leq\gamma_n')>1-\alpha$, where the latter implies that $d_k(P,\hat Q_n)\leq \gamma_n'$ is a level $\alpha$ test for testing $H_0: x^n\sim P$ and $H_1:x^n\sim Q$ with $P\neq Q$.  Together with Chernoff-Stein Lemma, we get
	\begin{align}
	&~\liminf_{n\to\infty}-\frac1n\log\mathbf P_{y^mx^n}\left(d_k(\hat P_m,\hat Q_n)\leq\gamma_{m,n}\right)\nn\\
	\leq &\liminf_{n\to\infty}-\frac1n\log \left(P\left(d_k(\hat P_m, P)\leq\gamma_{m}'\right) Q\left(d_k(P,\hat Q_n)\leq{\gamma_{n}'}\right)\right)\nn\\
	\leq&~\liminf_{n\to\infty}-\frac1n\log\left(1-{\alpha}\right)+\liminf_{n\to\infty}-\frac1n\log Q\left(d_k(P,\hat Q_n)\leq{\gamma_{n}'}\right)\nn\\
	\leq&~D(P\|Q).\nn
	\end{align}
	The proof is complete.
\end{proof}
\section{Proof of the Extended Sanov's Theorem}
\label{sec:extendedSanov}
Our proof is inspired by \cite{Csiszar2006simple} which proved the original Sanov's theorem w.r.t.~the $\tau$-topology. We first prove the result with a finite sample space and then extend it to the case with general Polish space. The prerequisites are two combinatorial lemmas that are standard tools in information theory.

For a positive integer $t$, let $\mathcal P_m(t)$ denote the set of probability distributions defined on $\{1,\ldots, t\}$ of form $P=\left(\frac{m_1}m,\cdots,\frac{m_t}m\right)$, with integers $m_1,\ldots, m_t$. Stated below are the two lemmas.
\begin{lemma}[{\cite[Theorem~11.1.1]{Cover2006}}]
	\label{lem:numEmpDistribution}
	$|\mathcal P_m(t)|\leq(m+1)^t.$
\end{lemma}
\begin{lemma}[{\cite[Theorem~11.1.4]{Cover2006}}]
	\label{lem:typeprob}
	Assume $y^m$ i.i.d.~$\sim R$ where $R$ is a distribution defined on $\{1,\ldots,t\}$. For any $P\in\mathcal P_m(t)$, the probability of the empirical distribution $\hat P_m$ of $y^m$ equal to $P$ satisfies	
	\[(m+1)^{-t}e^{-mD(P\|R)}\leq \mathbf P_{y^m}(\hat P_m=P)\leq e^{-mD(P\|R)}.\]
\end{lemma}
\subsection{Finite Sample Space}	
\paragraph{Upper bound} Let $t$ denote the cardinality of $\mathcal X$. Without loss of generality, assume that $\inf_{(R,S)\in\operatorname{int}\Gamma} cD(R\|P)+(1-c) D(S\|Q)<\infty$. Hence, the open set $\operatorname{int}\Gamma$ is non-empty. As $0<c=\lim_{m,n\to\infty}\frac{m}{m+n}<1$, we can find $m_0$ and $n_0$ such that there exists $(P'_m,Q'_n)\in\operatorname{int}\Gamma\cap \mathcal P_m(t)\times \mathcal P_m(t)$ for all $m>m_0$ and $n>n_0$, and that $cD(P_m'\|P)+(1-c)D(Q_n'\|Q)\to\inf_{(R,S)\in\operatorname{int}\Gamma} cD(R\|P)+(1-c) D(S\|Q)$ as $m,n\to\infty$. Then we have, with $m>m_0$ and $n>n_0$,
\begin{align}
\mathbf{P}_{y^mx^n}((\hat P_m, \hat Q_n)\in\Gamma) 
&=\sum_{(R,S)\in\Gamma\,\cap\, \mathcal{P}_{m}(t)\times\mathcal P_m(t)} \mathbf{P}_{y^mx^n}(\hat P_m=R, \hat Q_n=S)\nn\\
&\geq\sum_{(R,S)\in\operatorname{int}\Gamma\,\cap\, \mathcal{P}_{m}(t)\times\mathcal P_m(t)} \mathbf{P}_{y^mx^n}(\hat P_m=R, \hat Q_n=S)\nn\\
&\geq \mathbf{P}_{y^mx^n}(\hat P_m=P_m', \hat{Q}_n=Q_n')\nn\\
&= \mathbf P_{y^m}(\hat P_m=P'_m)\,\mathbf P_{x^n}(\hat Q_n=Q'_n)\nn\\
&\geq(m+1)^{-t}(n+1)^{-t}e^{-mD(P_m'\|P)}e^{-nD(Q_n'\|Q)}\nn,
\end{align}
where the last inequality is from Lemma~\ref{lem:typeprob}. It follows that \begin{align}
&~~~~\limsup_{m,n\to\infty}-\frac{1}{m+n}\log\mathbf{P}_{y^mx^n}((\hat{P}_m,\hat{Q}_n)\in\Gamma)\nn\\&\leq \lim_{m,n\to\infty}\frac1{m+n}\left(-t\log((m+1)(n+1))+mD(P'_m\|P)+nD(Q'_n\|Q)\right)\nn\\
&=\lim_{m,n\to\infty}\frac1{m+n}\left(mD(P'_n\|P)+nD(Q'_n\|Q)\right)\nn\\
&=\inf_{(R,S)\in\operatorname{int}\Gamma} cD(R\|P)+(1-c) D(S\|Q).\nn
\end{align}

\paragraph{Lower bound}
\begin{align}
\mathbf{P}_{y^mx^n}((\hat P_m, \hat Q_n)\in\Gamma) 
&= \sum_{(R,S)\in\Gamma\cap \mathcal{P}_{m}(t)\times\mathcal P_m(t)} \mathbf{P}_{y^m}(\hat P_m=R)\,\mathbf{P}_{x^n}(\hat Q_n=S)\nn\\
&\stackrel{(a)}{\leq} \sum_{(R,S)\in\Gamma\cap \mathcal{P}_m(t)\times\mathcal{P}_n(t)} e^{-mD(R\|P)}e^{-nD(S\|Q)}\nn\\
&\stackrel{(b)}{\leq} (m+1)^{t} (n+1)^{t}\sup_{(R,S)\in\Gamma} e^{-mD(R\|P)}e^{-nD(S\|Q)},
\end{align}
where $(a)$ and $(b)$ are due to  Lemma~\ref{lem:typeprob} and Lemma~\ref{lem:numEmpDistribution}, respectively. This gives \[\liminf_{m,n\to\infty}-\frac{1}{m+n}\log \mathbf{P}_{y^mx^n}((\hat{P}_m,\hat{Q}_n)\in\Gamma)\geq \inf_{(R,S)\in\Gamma} cD(R\|P)+(1-c)D(S\|Q),\]
and hence the lower bound by noting that $\Gamma\in\operatorname{cl}\Gamma$. Indeed, when the right hand side is finite, the infimum over $\Gamma$ equals the infimum over $\operatorname{cl}\Gamma$ as a result of the continuity of KLD for finite alphabets.

\subsection{Polish Sample Space}
We consider the general case with $\mathcal X$ being a Polish space. Now $\mathcal{P}$ is the space of probability measures on $\mathcal X$ endowed with the topology of weak convergence. To proceed, we introduce another topology on $\mathcal P$ and an equivalent definition of the KLD.

\paragraph{$\tau$-topology:} denote by $\Pi$ the set of all partitions $\mathcal A=\{A_1,\ldots, A_t\}$ of $\mathcal X$ into a finite number of measurable sets $A_i$. For $P\in\mathcal P$, $\mathcal A\in\Pi$, and $\zeta>0$, denote 
\begin{align}
\label{eqn:opentautoplogy}
U(P,\mathcal A, \zeta) = \{P'\in\mathcal P:|P'(A_i)-P(A_i)|<\zeta, i=1,\dots,t\}.
\end{align}
The $\tau$-topology on $\mathcal P$ is the coarsest topology in which the mapping $P\to P(F)$ are continuous for every measurable set $F\subset\mathcal X$. A base for this topology is the collection of the sets (\ref{eqn:opentautoplogy}). We will use $\mathcal P_\tau$ when we refer to $\mathcal P$ endowed with this $\tau$-topology, and write the interior and closure of a set $\Gamma\in\mathcal P_\tau$ as $\operatorname{int}_\tau\Gamma$
and $\operatorname{cl}_\tau\Gamma$, respectively. We remark that the $\tau$-topology is stronger than the weak topology: any open set in $\mathcal P$ w.r.t.~weak topology is also open in $\mathcal P_\tau$ (see more details in \cite{Csiszar2006simple,Dembo2009}). The product topology on $\mathcal P_\tau\times\mathcal P_\tau$ is determined by the base of the form of 
\[U(P,\mathcal A_1, \zeta_1)\times U(Q, \mathcal A_2, \zeta_2),\]
for $(P,Q)\in\mathcal P_\tau\times\mathcal P_\tau$, $\mathcal A_1,\mathcal A_2\in\Pi$, and $\zeta_1,\zeta_2>0$. We still use $\operatorname{int}_{\tau}(\Gamma)$ and $\operatorname{cl}_{\tau}(\Gamma)$ to denote the interior and closure of a set $\Gamma\subset\mathcal P_\tau\times\mathcal P_\tau$. As there always exists $\mathcal A\in\Pi$ that refines both $\mathcal A_1$ and $\mathcal A_2$, any element from the base has an open subset \[\tilde{U}(P,Q,\mathcal A,\zeta):=U(P,\mathcal A, \zeta)\times U(Q, \mathcal A, \zeta)\subset\mathcal P_\tau\times\mathcal P_\tau,\]
for some $\zeta >0$. 

\paragraph{Another definition of the KLD:} an equivalent definition of the KLD will also be used:
\begin{align}
D(P\|Q)=\sup_{\mathcal A\in\Pi} \sum_{i=1}^t P(A_i)\log\frac{P(A_i)}{Q(A_i)}=\sup_{\mathcal A\in\Pi}D(P^{\mathcal A}\|Q^{\mathcal A}),\nn
\end{align}
with the conventions $0\log 0=0\log\frac{0}{0}=0$ and $a\log\frac{a}{0}=+\infty$ if $a>0$. Here $P^{\mathcal A}$ denotes the discrete probability measure $(P(A_1),\ldots,P(A_t))$ obtained from probability measure $P$ and partition $\mathcal A$. It is not hard to verify that for $0<c<1$,
\begin{align}
\label{eqn:KLDdef}
cD(R\|P)+(1-c)D(S\|Q)&=c\sup_{\mathcal{A}_1\in\Pi}D(R^{\mathcal A_1}\|P^{\mathcal A_1})+(1-c)\sup_{\mathcal A_2\in\Pi}D(S^{\mathcal A_2}\|Q^{\mathcal A_2})\nn\\
&=\sup_{\mathcal{A}\in\Pi}\left(cD\left(R^{\mathcal A}\|P^{\mathcal A}\right)+(1-c)D\left(S^{\mathcal A}\|Q^{\mathcal A}\right)\right),
\end{align}
due to the existence of $\mathcal{A}$ that refines both $\mathcal{A}_1$ and $\mathcal A_2$ and the log-sum inequality \cite{Cover2006}.

We are ready to show the extended Sanov's theorem with Polish space.

\paragraph{Upper bound}
It suffices to consider only non-empty open $\Gamma$. If $\Gamma$ is open in $\mathcal P\times\mathcal P$, then $\Gamma$ is also open in $\mathcal P_\tau\times\mathcal P_\tau$. Therefore, for any $(R,S)\in\Gamma$, there exists a finite (measurable) partition $\mathcal A= \{A_1,\ldots,A_t\}$ of $\mathcal X$ and $\zeta>0$ such that 
\begin{align}
\label{eqn:opensubset}
\tilde{U}(R,S,\mathcal A,\zeta)=
\left\{(R',S'):|R(A_i)-R'(A_i)|<\zeta,|S(A_i)-S'(A_i)|<\zeta,i=1,\ldots,t\right\}\subset\Gamma.
\end{align}

Define the function $T:\mathcal X\to\{1,\ldots,t\}$ with $T(x)=i$ for $x\in A_i$. Then $(\hat P_m, \hat Q_n)\in\tilde{U}(R,S,\mathcal A,\zeta)$ with $R,S\in\Gamma$ if and only if the empirical measures $\hat P^{\circ}_m$ of $\{T(y_1),\ldots, T(y_m)\}:=T(y^m)$ and $\hat Q^{\circ}_n$ of $\{T(x_1),\ldots, T(x_n)\}:=T(x^n)$ lie in 
\[U^{\circ}(R,S,\mathcal A, \zeta)=\{(R^\circ,S^{\circ}):|R^{\circ}(i)-R(A_i)|<\zeta, |S^\circ(i)-S(A_i)|<\zeta,i=1,\ldots, t\}\subset \mathbb R^t\times\mathbb R^t.\]
Thus, we have
\begin{align}\mathbf{P}_{y^mx^n}((\hat P_m, \hat Q_n)\in\Gamma)&\geq\mathbf{P}_{y^mx^n}((\hat P_m, \hat Q_n)\in\tilde{U}(R,S,\mathcal A, \zeta))\nn\\
&=\mathbf{P}_{T(y^m)T(x^n)}((\hat P_m^{\circ}, \hat Q_n^{\circ})\in U^{\circ}(R,S,\mathcal A, \zeta)).\nn
\end{align}
As $T(x)$ and $T(y)$ takes values from a finite alphabet and $U^{\circ}(R,S,\mathcal A, \zeta)$ is open, we obtain that 
\begin{align}
&~\limsup_{m,n\to\infty}-\frac{1}{m+n}\log\mathbf{P}_{y^mx^n}((\hat P_m,\hat Q_n)\in\Gamma)\nn\\\leq&~ \limsup_{m,n\to\infty}-\frac{1}{m+n}\log\mathbf{P}_{T(y^m)T(x^n)}((\hat P_m^{\circ},\hat Q_n^{\circ})\in U^{\circ}(R,S,\mathcal A, \zeta))\nn\\
\leq&~\inf_{(R^\circ,S^\circ)\in U^{\circ}(R,S,\mathcal A, \zeta)} cD(R^\circ\|P^{\mathcal A})+(1-c)D(S^\circ\|Q^{\mathcal A})\nn\\
=&~\inf_{(R',S')\in\tilde{U}(R,S,\mathcal A, \zeta)} cD(R'^{\mathcal A}\|P^{\mathcal A})+(1-c)D(S'^{\mathcal A}\|Q^{\mathcal A})\nn\\
\leq&~ cD(R\|P)+(1-c)D(S\|Q),
\end{align}
where we have used definition of KLD in Eq.~(\ref{eqn:KLDdef})  and $(R,S)\in\tilde{U}(R,S,\mathcal A, \zeta)$ in the last inequality.  As $(R,S)$ is arbitrary in $\Gamma$, the lower bound is established by taking infimum over $\Gamma$.

\paragraph{Lower bound} With notations
\[\Gamma^{\mathcal A}=\{(R^{\mathcal{A}},S^\mathcal{A}):(R,S)\in\Gamma\},~ \Gamma(\mathcal A)=\{(R,S):(R^{\mathcal  A},S^{\mathcal{A}})\in\Gamma^{\mathcal{A}}\},\]
where $\mathcal A=\{A_1,\ldots,A_t\}$ is a finite partition, it holds that
\begin{align}
&~\mathbf{P}_{y^mx^n}((\hat P_m,\hat Q_n)\in\Gamma)\nn\\
\leq&~\mathbf P_{y^mx^n}((\hat P_m,\hat Q_n)\in\Gamma({\mathcal{A}}))\nn\\
=&~\mathbf P_{y^mx^n}((\hat P_m^{\mathcal A},\hat Q_n^{\mathcal A})\in\Gamma^{\mathcal A} \cap\mathcal{P}_{n}(t)\times{\mathcal P_m}(t))\nn\\
\leq&~(n+1)^t(m+1)^t\max_{(R^\circ,S^\circ)\in\Gamma^{\mathcal A}\cap\mathcal{P}_{n}(t)\times{\mathcal P_m(t)}}\mathbf P_{y^mx^n}\left(\hat{P}_n=R^\circ,\hat{Q}_m=S^\circ\right)\nn\\
\leq&~(n+1)^t(m+1)^t \exp\left(-\inf_{(R,S)\in\Gamma}  \left(nD(R^{\mathcal A}\|P^\mathcal{A})+mD(S^{\mathcal A}\|Q^\mathcal{A})\right)\right),\nn
\end{align}
where the last two inequalities are from Lemmas~\ref{lem:numEmpDistribution} and \ref{lem:typeprob}. As the above holds for any $\mathcal A\in\Pi$, Eq.~(\ref{eqn:KLDdef}) indicates
\begin{align}
&~\limsup_{m,n\to\infty}\frac{1}{m+n}\log\mathbf P_{y^mx^n}((\hat P_m,\hat Q_n)\in\Gamma)\nn\\
\leq&~\inf_{\mathcal{A}}\left(-\inf_{(R,S)\in\Gamma}  \left(cD(R^{\mathcal A}\|P^\mathcal{A})+(1-c)D(S^{\mathcal A}\|Q^\mathcal{A})\right)\right)\nn\\
=&~-\sup_{\mathcal{A}}\inf_{(R,S)\in\Gamma}  cD(R^{\mathcal A}\|P^\mathcal{A})+(1-c)D(S^{\mathcal A}\|Q^\mathcal{A}).\nn
\end{align}
Then the remaining of obtaining the lower bound is to show  
\[\sup_{\mathcal{A}}\inf_{(R,S)\in\Gamma} cD(R^{\mathcal A}\|P^\mathcal{A})+(1-c)D(S^{\mathcal A}\|Q^\mathcal{A})\geq \inf_{(R,S) \in\operatorname{cl}\Gamma}  cD(R\|P)+(1-c)D(S\|Q).\]




Assuming, without loss of generality, that the left hand side is finite, we only need to show

\[\operatorname{cl}\Gamma\cap B(P,Q,\eta)\neq\varnothing,\]
whenever \[\eta>\sup_{\mathcal{A}}\inf_{(R,S)\in\Gamma} cD(R^{\mathcal A}\|P^\mathcal{A})+(1-c)D(S^{\mathcal A}\|Q^\mathcal{A}).\] Here $B(P,Q,\eta)$ is the divergence ball defined as follows
\[B(P,Q,\eta)=\left\{(R,S):cD(R\|P)+(1-c)D(S\|Q)\leq\eta\right\},\]
which is compact in $\mathcal P\times\mathcal P$~w.r.t.~the weak topology, due to the lower semi-continuity of $D(\cdot\|P)$ and $D(\cdot\|Q)$ as well as the fact that $0<c<1$.

To this end, we first show the following:
\begin{align}
\label{eqn:clGamma}
\operatorname{cl}\Gamma=\bigcap_{\mathcal{A}}\operatorname{cl}\Gamma(\mathcal{A}).
\end{align}
The inclusion is obvious since $\Gamma\in\Gamma(\mathcal{A})$. The reverse means that if $(R,S)\in \operatorname{cl}\Gamma(\mathcal A)$ for each $\mathcal{A}$, then any neighborhood of $(R,S)$ w.r.t.~the weak convergence intersects $\Gamma$. To verify this, let $O(R,S)$ be a neighborhood of $(R,S)$ w.r.t.~the weak convergence, then there exists $\tilde{U}(R,S,\mathcal B,\zeta)\in O(R,S)$ over a finite partition $\mathcal B$ as $O(R,S)$ is also open in $\mathcal P_\tau\times\mathcal P_\tau$. Furthermore, the partition $\mathcal B$ can be chosen to refine $\mathcal A$ so that $\operatorname{cl}\Gamma(\mathcal B)\subset\operatorname{cl}\Gamma(\mathcal A)$. As $\tau$-topology is stronger than the weak topology,  a closed set in the $\mathcal P_\tau\times\mathcal P_\tau$ is closed in $\mathcal P\times\mathcal P$, and hence $\operatorname{cl}\Gamma(\mathcal B)\subset \operatorname{cl}_{\tau} \Gamma(\mathcal B)$. That $(R,S)\in\operatorname{cl}_\tau\Gamma(\mathcal B)$ implies that there exists $(R',S')\in\tilde{U}(R,S,\mathcal B,\zeta)\cap\Gamma(\mathcal B)$. By the definition of $\Gamma(\mathcal B)$, we can also find $(\tilde{R},\tilde{S})\in\Gamma$ such that $\tilde{R}(B_i)=R'(B_i)$ and $\tilde{S}(B_i)=S'(B_i)$ for each $B_i\in\mathcal B$, and hence  $(\tilde{R},\tilde{S})\in\tilde{U}(R,S,\mathcal B, \zeta)$. In summary, we have $(\tilde{R},\tilde{S})\in\tilde{U}(R,S,\mathcal B,\zeta)\subset O(R,S)$ and  $(\tilde{R},\tilde{S})\in\Gamma$. Therefore, $\Gamma\cap O(R,S)\neq\varnothing$ and the claim follows.

Next we show that, for each partition $\mathcal A$, 
\begin{align}
\Gamma(\mathcal A)\cap B(P,Q,\eta)\neq\varnothing.
\end{align}
By Eq.~(\ref{eqn:KLDdef}), there exists $(\tilde{P},\tilde Q)$ such that 
$cD(\tilde{P}^{\mathcal A}\|P^{\mathcal A})+(1-c)D(\tilde{Q}^{\mathcal A}\|Q^{\mathcal A})\leq\eta$. For such $(\tilde P, \tilde Q)$, we can construct $(P',Q')\in\Gamma(\mathcal A)$ as 
\begin{align}
P'(F)&=\sum_{i=1}^t\frac{\tilde{P}(A_i)}{P(A_i)}P(F\cap A_i),\nn\\
Q'(F)&=\sum_{i=1}^t\frac{\tilde{Q}(A_i)}{Q(A_i)}Q(F\cap A_i),\nn
\end{align}
for any measurable subset $F\subset\mathcal X$. If $P(A_i)=0$ ($Q(A_i)=0$) and hence $\tilde{P}(A_i)=0$ ($\tilde{Q}(A_i)=0$), as $D(\tilde P^{\mathcal A}\|P^{\mathcal A})<\infty$ ($D(\tilde Q^{\mathcal A}\|Q^{\mathcal A})<\infty$), for some $i$, the corresponding term in the above equation is set equal to $0$. Then $(P',Q')$ belongs to $\Gamma(\mathcal A)$ and also lies in $B(P,Q,\eta)$. The latter is because $D(P'\|P)=D(\tilde{P}^{\mathcal{A}}\|Q^{\mathcal A})$ and $D(Q'\|Q)=D(\tilde{Q}^{\mathcal A}\|Q^{\mathcal A})$: one can verify that any $\mathcal B$ that refines $\mathcal A$ satisfies \[D({P}'^{\mathcal B}\|P^{\mathcal B})=D(\tilde{P}^{\mathcal A}\|P^{\mathcal A}), D({Q}'^{\mathcal B}\|Q^{\mathcal B})=D(\tilde{Q}^{\mathcal A}\|Q^{\mathcal A}).\]

For any finite collection of partitions $\mathcal A_i\in\Pi$ and $\mathcal A\in\Pi$ refining each $\mathcal A_i$, each $\Gamma(\mathcal A_i)$ contains $\Gamma(\mathcal A)$. This implies that 
\[\bigcap_{i=1}^r\left(\Gamma(\mathcal A_i)\cap B(p,q,\eta)\right)\neq\varnothing,\]
for any finite $r$. Finally, the set $\operatorname{cl}\Gamma(\mathcal A)\cap B(P,Q,\eta)$ for any $\mathcal A$ is compact due to the compactness of $B(P,Q,\eta)$, and any finite collection of them has non-empty intersection. It follows that all these sets is also non-empty. This completes the proof.

\section{Proof of Theorem~\ref{thm:mainresult1}}
\label{sec:proofmainresult}
\begin{proof} 
According to Theorem~\ref{thm:MMDmetrize}, $d_k$ metrizes the weak convergence over $\mathcal P$. For convenience, we will write the type-I and type-II error probabilities as $\alpha_{m,n}$ and  $\beta_{m,n}$, respectively; we will also use $\beta$ to denote the type-II error exponent. That $\alpha_{m,n}\leq\alpha$ is clear from Lemma~\ref{lem:gammanm}, and we only need to show that $\beta_{m,n}$ vanishes exponentially as $m$ and $n$ scale. 
	
	We first show $\beta\geq D^*$. With a fixed $\gamma>0$, we have $\gamma_{m,n}\leq \gamma$ for sufficiently large $n$ and $m$. Therefore,
	\begin{align}	
	\label{eqn:eqn1}
	\beta&=\liminf_{m,n\to\infty}-\frac{1}{m+n}\log \mathbf{P}_{y^mx^n}(d_k(\hat P_m, \hat Q_n)\leq\gamma_{m,n})\nn\\
	&\geq\liminf_{m,n\to\infty}-\frac{1}{m+n}\log\mathbf{P}_{y^mx^n}(d_k(\hat P_m, \hat Q_n)\leq\gamma)\nn\\
	&\geq\inf_{(R,S):d_k(R,S)\leq\gamma}cD(R\|P)+(1-c)D(S\|Q)\nn\\
	&:=D_\gamma^*,
	\end{align}
	where the last inequality is from the extended Sanov's theorem and that $d_k$ metrizes weak convergence of $\mathcal P$ so that $\{(R,S):d_k(R,S)\leq\gamma\}$ is closed in the product topology on $\mathcal P\times\mathcal P$. Since $\gamma>0$ can be arbitrarily small, we have 
	\[\beta\geq\lim_{\gamma\to 0^{+}}D^*_\gamma,\]
	where the limit on the right hand side must exist as $D^*_{\gamma}$ is positive, non-decreasing when $\gamma$ decreases, and bounded by $D^*$ that is assumed to be finite. Then it suffices to show 
	\begin{align}
	\lim_{\gamma\to0^{+}} D_\gamma^*=D^*.\nn
	\end{align}

	To this end, let $(R_\gamma, S_\gamma)$ be such that  $d_k(R_\gamma,S_\gamma)\leq\gamma$ and $cD(R_\gamma\|P)+(1-c)D(S_\gamma\|Q)=D^*_\gamma$. Notice that $R_\gamma$ and $S_\gamma$ must lie in
	\[\left\{W:D(W\|P)\leq\frac{D^*}c,  D(W\|Q)\leq\frac{D^*}{1-c}\right\}:= \mathcal W,\]
	for otherwise $D_\gamma^*>D^*$. We remark that $\mathcal W$ is a compact set in $\mathcal P$ as a result of the lower semi-continuity of KLD w.r.t.~the weak topology on $\mathcal P$ \cite{VanErven2014RenyiKLD,Dembo2009}. Existence of such a pair can be seen from the facts that $\{(R,S):d_k(R,S)\leq \gamma\}$ is closed and convex, and that both $D(\cdot{\|P})$ and $D(\cdot\|Q)$ are convex functions \cite{VanErven2014RenyiKLD}.
	
	Assume that $D^*$ cannot be achieved. We can write 
	\begin{align}
	\label{eqn:assu}
	\lim_{\gamma\to0^+}D^*_{\gamma}=D^*-\epsilon,
	\end{align}
	for some $\epsilon>0$. By the definition of lower semi-continuity, there exists a $\kappa_W>0$ for each $W\in\mathcal W$ such that 
	\begin{align}
	\label{eqn:geqhalf}
	cD(R\|P)+(1-c)D(S\|Q)\geq cD(W\|P)+(1-c)D(W\|Q)-\frac\epsilon 2	\geq  D^*-\frac\epsilon 2,
	\end{align}
	whenever $R$ and $S$ are both from 
	\[\mathcal S_W=\left\{R:d_k(R,W)<\kappa_W\right\}.\]
	Here the last inequality  comes from  the definition of $D^*$ given in Theorem~\ref{thm:mainresult1}. To find a contradiction, define 
	\[\mathcal S_W'=\left\{R:d_k(R,W)<\frac{\kappa_W}{2}\right\}.\]
	Since $S_W'$ is open and $\bigcup_W\mathcal S_{W}'$ covers $\mathcal W$, the compactness of $\mathcal W$ implies that there exists finite $\mathcal S_W'$'s, denoted by $\mathcal S_{W_1}',\ldots,\mathcal S_{W_N}'$, covering $\mathcal W$. Define $\kappa^*=\min_{i=1}^N\kappa_{W_i}>0$. Now let $\gamma<{\kappa^*}/{2}$ as $\gamma$ can be made arbitrarily small. Since $\bigcup_{i=1}^N \mathcal S'_{W_i}$ covers $\mathcal W$, we can find a $W_i$ with $R_\gamma\in \mathcal S_{W_i}'\subset\mathcal S_{W_i}$. Thus, it holds that \[d_k(S_\gamma,W_i)\leq d_k(S_\gamma,R_\gamma)+d_k(R_\gamma,W_i)<\kappa_{W_i}.\]
	That is, $S_\gamma$ also lies in $\mathcal S_{W_i}$. By Eq.~(\ref{eqn:geqhalf}) we get
	\[cD(R_\gamma\|P)+(1-c)D(S_\gamma\|Q)\geq D^*-\epsilon/2.\]
	However, by our assumption in Eq.~(\ref{eqn:assu}), it should hold that
	\[cD(R_\gamma\|P)+(1-c)D(S_\gamma\|Q)\leq D^*-\epsilon.\]
	Therefore, $\beta\geq D^*$.
	
	The other direction can be simply seen from the optimal type-II error exponent in Theorem~\ref{thm:upperbound}. Alternatively, we can use Chernoff-Stein lemma in a similar manner to the proof of Theorem~\ref{thm:simple1}. Let $P'$ be such that $cD(P'\|P)+(1-c)D(P'\|Q)=D^*.$
	Such $P'$ exists because $0<D^*<\infty$ and $D(\cdot\|P)$ and $D(\cdot\|Q)$ are convex w.r.t.~$\mathcal P$. That $D^*$ is bounded implies that both $D(P'\|P)$ and $D(P'\|Q)$ are finite. We  have 
	\begin{align}
	\beta_{m,n}=&~\mathbf P_{y^mx^n}(d_k(\hat P_m,\hat Q_n)\leq \gamma_{m,n})\nn\\
	\stackrel{(a)}{\geq}&~\mathbf P_{y^mx^n}(d_k(P',\hat P_m)+d_k(P',\hat Q_n)\leq{\gamma_{m,n}})\nn\\
	\stackrel{(b)}{\geq}&~\mathbf P_{y^mx^n}(d_k(P',\hat P_m)\leq\gamma_{m}, d_k(P',\hat Q_m)\leq{\gamma_{n}})\nn\\
	=&~P(d_k(P',\hat P_m)\leq\gamma_{m})\,Q(d_k(P',\hat Q_n)\leq{\gamma_{n}}),\nn
	\end{align}
	where $(a)$ and $(b)$ are from the triangle inequality of the metric $d_k$, and we pick $ \gamma_n=\sqrt{2K/n}(1+\sqrt{-\log\alpha})$, and  $\gamma_m=\sqrt{2K/m}(1+\sqrt{-\log\alpha})$ so that $\gamma_{m,n}>\gamma_n+\gamma_m$. Then Lemma~\ref{lem:gamman} implies  $P'(d_k(P',\hat P_m)\leq\gamma_m)> 1-\alpha$. For now assume that $D(P'\|P)> 0$ and $D(P'\|Q)>0$. We can regard $\{y^m:d_k(P',\hat P_m)\leq\gamma_m\}$ as an acceptance region for testing $H_0:y^m\sim P'$ and $H_1:y^m\sim P$. Clearly, this test performs no better than the optimal level $\alpha$ test for this simple hypothesis testing in terms of the type-II error probability. Therefore, Chernoff-Stein lemma implies 
	\begin{align}
	\label{eqn:PP}
	\liminf_{m\to\infty}-\frac{1}{m}\log P(d_k(P',\hat P_m)\leq\gamma_m)\leq D(P'\|P).
	\end{align}
	Analogously, we have
	\begin{align}
	\label{eqn:QQ}
	\liminf_{n\to\infty}-\frac{1}{n}\log Q(d_k(P', \hat{Q}_n)\leq\gamma_n)\leq D(P'\|Q).
	\end{align}
	
	Now assume without loss of generality that $D(P'\|P)=0$, i.e., $P'=P$. Then $D(P'\|Q)>0$ under the alternative hypothesis $H_1:P\neq Q$, and Eq.~(\ref{eqn:QQ}) still holds. Using Lemma~\ref{lem:gamman}, we have $P(d_k(P',\hat P_m)\leq\gamma_m)>1-\alpha$, which gives zero exponent. Therefore, Eq.~(\ref{eqn:PP}) holds with $P'=P$.
	
	As $\lim_{m,n\to\infty}\frac{m}{m+n}=c$, we conclude that	\[\beta=\liminf_{m,n\to\infty}-\frac{1}{m+n}\log\beta_{m,n}\leq D^*.\]
	The proof is complete.
\end{proof}

\section{Proof of Theorem~\ref{thm:upperbound}}
\label{sec:optimality}
\begin{proof}
	Let $P'$ be such that $cD(P'\|P)+(1-c)D(P'\|Q)=D^*$. Consider first $D(P'\|P)\neq0$ and $D(P'\|Q)\neq0$. Since $D^*$ is assumed to be finite, we have both $D(P'\|P)$ and $D(P'\|Q)$ being finite. This implies that $P'$ is absolutely continuous w.r.t.~both $P$ and $Q$, so the Radon-Nikodym derivatives ${dP'}/{dP}$ and ${dP'}/{dQ}$ exist. 
	
	Define two sets 
	\begin{equation}
	\begin{aligned}
	\label{eqn:kldset}
	A_m &= \left\{y^m:D(P'\|P)-\epsilon\leq\frac{1}{m}\log\frac{dP'(y^m)}{dP(y^m)}\leq D(P'\|P)+\epsilon\right\},\\
	B_n &= \left\{x^n:D(P'\|Q)-\epsilon\leq\frac{1}{n}\log\frac{dP'(x^n)}{dQ({x^n})}\leq D(P'\|Q)+\epsilon\right\},
	\end{aligned}
	\end{equation}
	Recall the definition of the KLD: $D(P'\|P)=\mathbf{E}_{x\sim P'}\log(dP'(x)/dP(x))$ and $D(P'\|Q)=\mathbf{E}_{x\sim P'}\log(dP'(x)/dQ(x))$. By law of large numbers,  we have for any given $\epsilon>0$,
	\begin{align}
	\label{eqn:kldsetlargeProb}
	\mathbf P_{y^mx^n}(A_m\times B_n)\geq 1-\epsilon,~\text{for large enough}~m~\text{and}~n,
	\end{align}
	with $y^m$ and $x^n$ i.i.d.~$\sim P'$. 
	
	Now consider the type-II error probability of level $\alpha$ tests. First, for a level $\alpha$ test, we have its acceptance region satisfies
	\begin{align}
	\label{eqn:anyalphatestLargeProb}
	\mathbf P_{y^mx^n}(\Omega_0'(m,n))> 1-\alpha,
	\end{align}
	when  $y^m$ and $x^n$ i.i.d.~$\sim P'$, i.e., when the null hypothesis $H_0: P=Q$ holds. Then under the alternative hypothesis $H_1:P\neq Q$, we have 
	\begin{align}
	&~\mathbf P_{y^mx^n} (\Omega_0'(m,n))\nn\\
	\geq&~\mathbf P_{y^mx^n} (A_m\times B_n\cap\Omega_0'(m,n))\nn\\
	=&~\int_{A_m\times B_n\cap\Omega_0'(m,n)} dP(y^m)\,dQ(x^n)\nn\\
	\stackrel{(a)}{\geq}&~\int_{A_m\times B_n\cap\Omega_0'(m,n)}2^{-m(D(P'\|P)+\epsilon)}2^{-n(D(P'\|Q)+\epsilon)} dP'(y^m)\,dP'(x^n)\nn\\
	=&~2^{-mD(P'\|P)-n(D(P'\|Q)-(m+n)\epsilon} \int_{A_m\times B_n\cap\Omega_0'(m,n)}dP'(y^m)\,dP'(x^n)\nn\\
	\stackrel{(b)}{\geq}&~2^{-mD(P'\|P) -nD(P'\|Q)-(m+n)\epsilon}(1-\alpha-\epsilon),\nn
	\end{align}
	where $(a)$ is from Eq.~(\ref{eqn:kldset}) and $(b)$ is due to Eqs.~(\ref{eqn:kldsetlargeProb}) and (\ref{eqn:anyalphatestLargeProb}). Thus, when $\epsilon$ is small enough so that $1-\alpha-\epsilon>0$, we get 
	\begin{align}
	\label{eqn:upp}
	\liminf_{m,n\to\infty}-\frac{1}{m+n}\log\mathbf P_{y^mx^n} (\Omega_0'(m,n))&\leq \liminf_{m,n\to\infty}-\frac{1}{m+n}\left(mD(P'\|P)+n(D(P'\|Q)+(m+n)\epsilon\right)\nn\\&=D^*+\epsilon.
	\end{align}
	
	If a test is an asymptotic level $\alpha$ test, we can replace $\alpha$ by $\alpha+\epsilon'$ where $\epsilon'$ can be made arbitrarily small provided that $m$ and $n$ are large enough. Thus, Eq.~(\ref{eqn:upp}) holds too. Finally, since $\epsilon$ can also be arbitrarily small, we conclude that \[\liminf_{m,n\to\infty}-\frac{1}{m+n}\log\mathbf P_{y^mx^n} (\Omega_0'(m,n))\leq D^*.\]		
	If $P'=P$, then $A_m$ contains all $y^m\in\mathcal X^m$ and the above procedure gives the same result. 
\end{proof}

\section{Experiments}
\label{sec:exp}
This section presents empirical results of the MMD and KSD based goodness-of-fit tests in the finite sample regime. We note that there have been extensive experiments in \cite{Chwialkowski2016Goodness,Gorham2017measuringKSD,Liu2016GoodnessFit,Jitkrittum2017linearGoodness} and the sample size $m$ drawn from $P$ is usually fixed for the kernel two-sample test. As such, we only consider two toy experiments and let $m$  scale as required in Theorem~\ref{thm:2sample1}. 

We evaluate the following tests with a fixed level $\alpha=0.1$, all using Gaussian kernel $k(x,y)=e^{-\|x-y\|_2^2/(2w)}$: {1)~{\tt{Simple}}:} the simple kernel test $d_k(P,\hat Q_n)$. The acceptance threshold is estimated by drawing i.i.d.~samples from $P$, i.e., the Monte Carlo method. The number of trials is $500$. 2)~{\tt Two-sample}: the two-sample test $d_k(\hat P_m, \hat Q_n)$ with $m=n^{1.5}$. Threshold is obtained from the bootstrap method in \cite{Gretton2012}, with $500$ bootstrap replicates. 3)~{\tt{KSD}}: the KSD based test $d_S^2(P, \hat Q_n)$. We use wild bootstrap method from \cite{Chwialkowski2016Goodness} with $500$ replicates to estimate the $\alpha$-quantile.

{\bf Gaussian vs.~Laplace.} We use a similar experiment setting in \cite{Jitkrittum2017linearGoodness}. Consider $P:\mathcal N(0, 2\sqrt2)$ and $Q:\operatorname{Laplace} (0,2)$, a zero-mean Laplace distribution with scale parameter $2$. The parameters are chosen so that $P$ and $Q$ have the same mean and variance. We pick a fixed bandwidth $w=1$ for all the kernel based tests and repeat $500$ trials of each sample size $n$ for both hypotheses. We also evaluate the likelihood ratio test {\tt{LR}}, an oracle approach assuming both $P$ and $Q$ are known. In Figure~\ref{fig:pes2}, {\tt LR} has the lowest type-II error probabilities as expected, while {\tt Simple} and {\tt Two-sample} perform slightly better than {\tt KSD}.  As shown in Figure~\ref{fig:pes1}, all the kernel based tests have the type-I error probabilities around the given level $\alpha=0.1$, except for {\tt KSD} with $n=5$ samples.
\begin{figure}[ht]
	\centering
	\begin{subfigure}[b]{0.4\textwidth}
		\includegraphics[width=\textwidth]{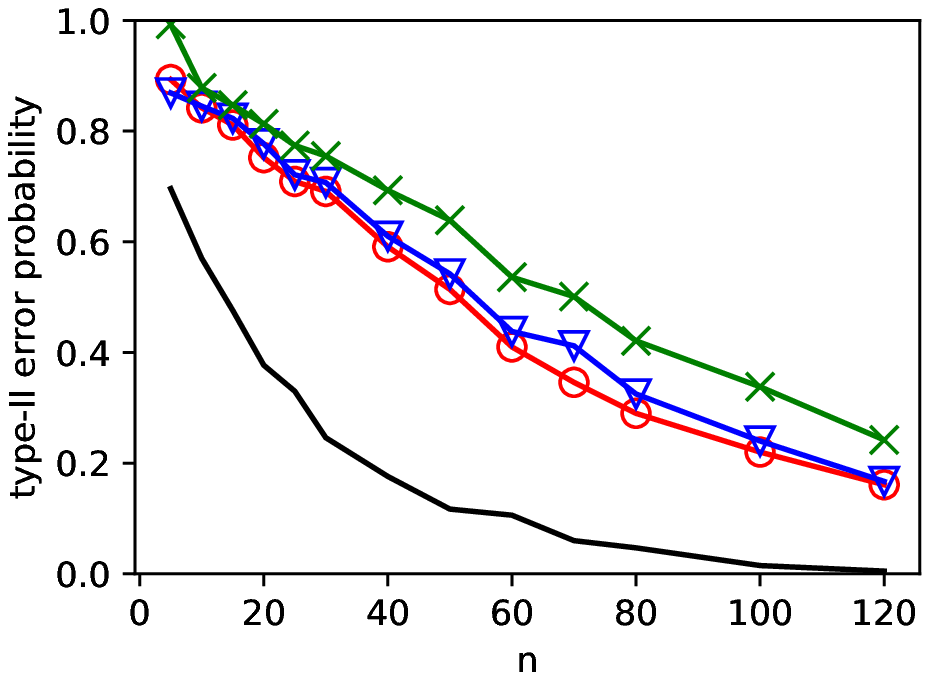}		
		\caption{type-II error}
		\label{fig:pes2}
	\end{subfigure}
	\begin{subfigure}[b]{0.4\textwidth}
		\includegraphics[width=\textwidth]{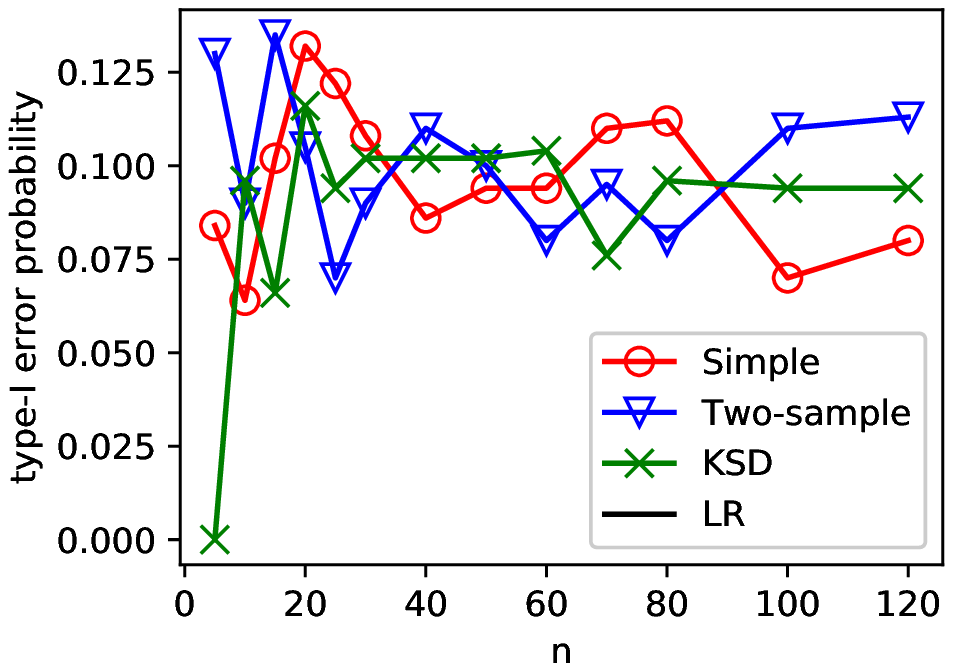}
		\caption{type-I error}
		\label{fig:pes1}
	\end{subfigure}
	\caption{Gaussian vs.~Laplace.}
\end{figure}

{\bf Gaussian Mixture.} The next experiment is taken from \cite{Liu2016GoodnessFit}. The i.i.d.~observations $x^n$ are drawn from $Q:\sum_{i=1}^5 a_i\,\mathcal N(x;\mu_i,\sigma^2)$ with $a_i=1/5$, $\sigma^2=1$, and $\mu_i$ randomly drawn from $\operatorname{Uniform}[0,10]$. We then generate $P$ by adding standard Gaussian noise (perturbation) to $\mu_i$. In \citep{Liu2016GoodnessFit}, the sample number $m$ drawn from $P$ is fixed while the observed sample number $n$ varies. We report the type-II error probabilities in Figure~2, averaged over $500$ random trials.
\begin{figure}[ht]
	\centering
	\begin{subfigure}[b]{0.4\textwidth}
		\includegraphics[width=\textwidth]{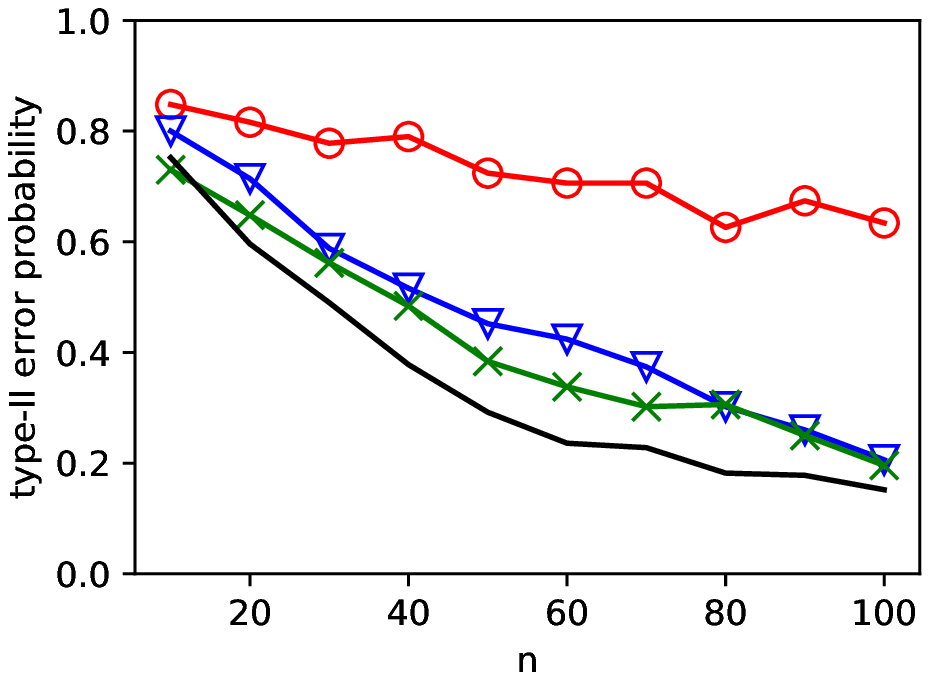}
		\caption{type-II error}
		\label{fig:MOG1}
	\end{subfigure}
	\begin{subfigure}[b]{0.4\textwidth}
		\includegraphics[width=\textwidth]{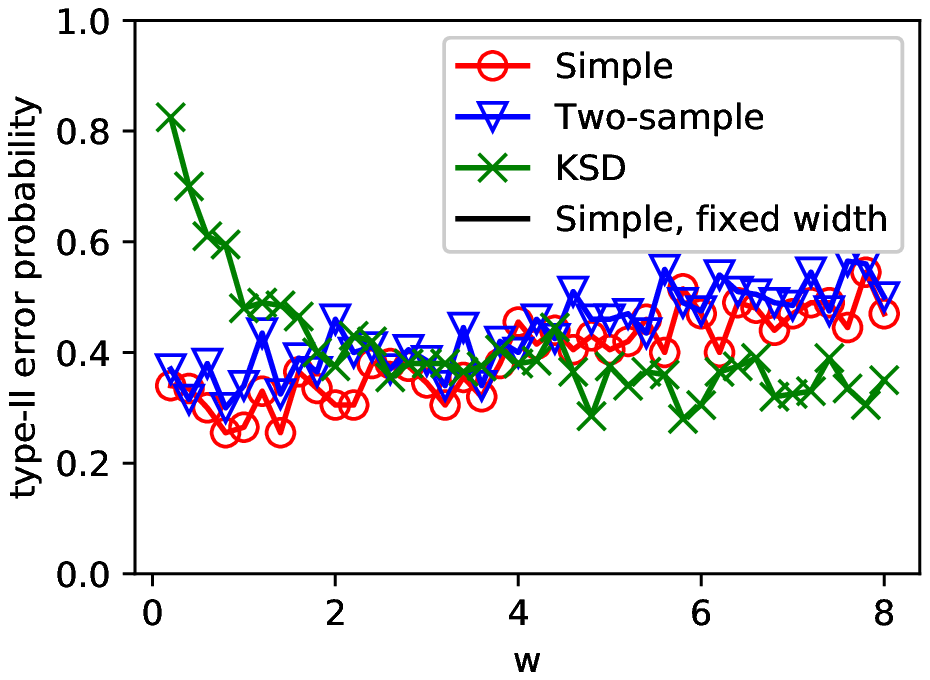}
		\caption{type-II error}
		\label{fig:MOG2}
	\end{subfigure}
	\centering
	\caption{Gaussian mixture. (a) median bandwidth for {\tt Simple}, {\tt Two-sample}, and {\tt KSD}, and a fixed bandwidth $w = 1$ for
		{\tt Simple}; (b) fixing $n = 50$ samples and varying kernel bandwidths.}
	\label{fig:MOG}
\end{figure}

With the median heuristic for bandwidth choice, {\tt KSD} and {\tt Two-sample} perform similarly whereas {\tt Simple} has its type-II error probability decreasing slowly, as shown in Figure~\ref{fig:MOG1}. Picking a fixed bandwidth $w=1$ for {\tt Simple} again results in a better performance. In light of the role of kernels, we then search over the kernel bandwidths in $[0, 8]$ for a fixed sample size $n=50$. In Figure~\ref{fig:MOG2}, {\tt Simple} and {\tt Two-sample} tend to achieve lower type-II error probabilities when $w$ is small, while {\tt KSD} has  a lower error probability around $w=5$. The optimal type-II error probabilities of {\tt Simple} and {\tt KSD} are close and slightly lower than that of {\tt Two-sample}. While computational issue is not the focus of this paper, we do observe that {\tt KSD} is more efficient in this experiment, as it does not need to draw samples.

Whereas we cannot tell much statistical difference in our experiments, some experiments in the literature showed that the MMD based tests performed better than the KSD based tests and others showed the opposite \cite{Chwialkowski2016Goodness,Gorham2017measuringKSD,Liu2016GoodnessFit,Jitkrittum2017linearGoodness}. The finite sample performance depends on kernel choice as well as specific distributions. Under the universal setting, no test is known to be optimal in terms of the type-II error probability subject to a given level constraint. Statistical optimality can only be established in the large sample limit, as the one considered in the present work.



\end{document}